%% file: main.tex
\documentclass[letterpaper]{article} 
\PassOptionsToPackage{table}{xcolor}
\usepackage[preprint]{aaai2027}
\usepackage[hyphens]{url}  
\usepackage{graphicx} 
\usepackage{natbib}  
\usepackage{caption} 
\usepackage{algorithm}
\usepackage{algorithmic}
\usepackage{booktabs}
\usepackage{multirow}
\usepackage{adjustbox}
\usepackage{subcaption}
\usepackage{amsmath}
\usepackage{amssymb}
\usepackage{amsthm}

\definecolor{avgshade}{HTML}{72B879}
\definecolor{passshade}{HTML}{6FA8DC}

\newcommand{\Ex}{\mathbb{E}}
\newcommand{\ind}{\mathbb{1}}
\newcommand{\method}{RouteOPD}
\newcommand{\methodfull}{Routed On-Policy Distillation}
\newtheorem{proposition}{Proposition}

\title{Distillation as Probability Transport: Routed On-Policy Distillation}
\author{
Tianle Xia$^{*}$, Lingxiang Hu$^{*}$, Yiding Sun, Linfang Shang, Ming Xu \\
Lan Xu$^{\dagger}$, Ning Zheng, Wei Xu, Jie Jiang
}
\affiliations{
Tencent \\
\texttt{\{tianlexia,lingxianghu,emanuelsun,faelynshang,flemingxu\}@tencent.com} \\
\texttt{\{lanxu,yodazheng,davidxu,zeus\}@tencent.com} \\
{\small $^{*}$First author. \quad $^{\dagger}$Corresponding author.}
}

\begin{document}
\maketitle

\begin{abstract}
On-policy distillation (OPD) transfers teacher knowledge on student-generated trajectories, but
efficient sampled objectives reduce the teacher distribution to scalar credit on individual
tokens. Such credit indicates whether a token should gain or lose probability, yet leaves the
corresponding redistribution unspecified. We recast OPD as teacher-guided probability transport
and propose \method{} (\methodfull{}), which decomposes local teacher--student disagreement into
student-excess sources and teacher-deficit destinations and couples them into explicit transport
pairs. \method{} optimizes pairwise log-odds toward jointly realizable targets obtained from a
bounded teacher potential, while adapting the transport budget to the concentration of teacher
demand. This formulation directs updates toward teacher-preferred destinations and controls their
magnitude within a single transport operator. Experiments across four teacher--student settings
and four mathematical-reasoning benchmarks demonstrate that \method{} consistently outperforms
sampled reverse-KL OPD, with improvements accompanied by higher routing fidelity and lower
background leakage. These results demonstrate the effectiveness of explicitly modeling
probability transport in on-policy distillation.
\end{abstract}


\section{Introduction}

On-policy distillation (OPD) has emerged as an effective post-training paradigm for transferring
reasoning capabilities between language models \citep{agarwal2024opd,qwen3}. Unlike offline
distillation, which supervises a student on fixed teacher or corpus trajectories
\citep{hinton2015distilling,kim2016sequence}, OPD samples from the student and queries the teacher
at the resulting states. It therefore combines on-policy state coverage with
\emph{temporally dense} teacher feedback, reducing the state-distribution mismatch induced by
fixed histories \citep{ross2011reduction,bengio2015scheduled} without relying only on sparse
outcome rewards \citep{xia2026searchp1}. This combination is especially attractive for
long-horizon mathematical reasoning, where an early mistake changes every subsequent state and
final-answer feedback gives little information about which local decisions should be corrected.
Yet this advantage raises a
more basic question: when the teacher prescribes a categorical distributional correction, is
scalar credit on an individual token sufficient to specify how the student policy should change?

\begin{figure}[t]
  \centering
  \includegraphics[width=\columnwidth]{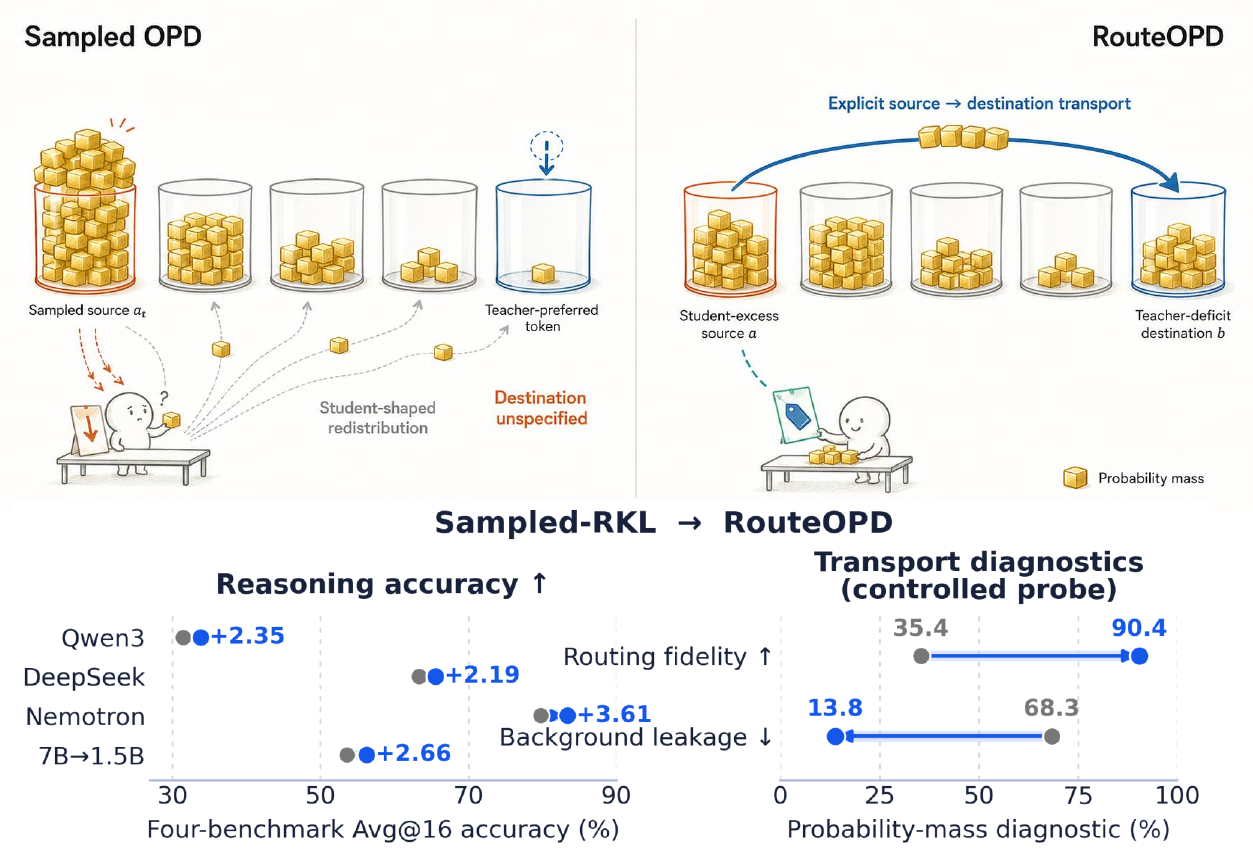}
  \caption{\textbf{\method{} replaces implicit redistribution with teacher-guided transport,}
  improving average reasoning accuracy and routing fidelity while reducing leakage.}
  \label{fig:routeopd-teaser}
\end{figure}

An efficient and common instantiation of OPD, however, leaves each update structurally
under-specified. Although the teacher provides a categorical distribution over the vocabulary, a
sampled objective exposes one student action and one scalar credit at a visited state. This scalar
can indicate whether the sampled token should be encouraged or suppressed, and with what strength.
Yet a normalized categorical policy cannot change one probability in isolation: suppressing one
token releases mass that must be assigned elsewhere, while increasing one token must draw mass
from other alternatives. Thus, \emph{scalar credit is not a transport plan}. It explicitly
identifies at most one side of a redistribution, even though the teacher distribution provides
evidence about both student-excess sources and teacher-deficit destinations.

This distinction matters because the missing route is not neutral: it is filled in by student
softmax geometry. For logits $z$ and a sampled token $a$,
$\nabla_z\log\pi_\theta(a\mid s)=\mathbf e_a-\pi_\theta(\cdot\mid s)$. Consequently, a negative
credit suppresses $a$ but produces a direct background logit update over every non-source token,
weighted by the student's current probabilities. Even when a scalar estimator recovers a dense
objective in expectation, each realized stochastic update still delegates this unspecified
redistribution to the student geometry. It can therefore favor alternatives that are already
common under the student but exhibit no teacher deficit. Sampled OPD may estimate how strongly to
correct a realized action while remaining incomplete about which alternatives should receive that
correction. This per-update under-specification motivates explicit control over both sides of the
redistribution.

We introduce \method{} (\methodfull{}), illustrated in
Figure~\ref{fig:routeopd-overview}, to replace implicit redistribution with explicit,
teacher-guided probability transport. \method{} treats teacher--student disagreement as a local
conservation problem: probability removed from student-excess tokens should be reassigned to
teacher-deficit tokens. At each student-visited state, it constructs a detached union of teacher
and behavior-student top-$k$ tokens, then decomposes the normalized disagreement into excess source
mass and deficit destination mass. Because both distributions are normalized on the common union,
the two masses have equal totals; normalizing them yields the source and destination marginals.
We connect these marginals with their independent product coupling---the maximum-entropy coupling
under fixed marginals. Sampling this coupling produces explicit source--destination routes without
materializing a dense quadratic coupling.

For each sampled route $(a,b)$, \method{} learns the teacher-implied change in pairwise log-odds.
The identity
$\nabla_z[\log\pi_\theta(b\mid s)-\log\pi_\theta(a\mid s)]
=\mathbf e_b-\mathbf e_a$
cancels the shared log-softmax background and gives zero \emph{direct} gradient to unrelated
logits. The realized direct logit update is therefore localized to the specified source and
destination rather than spread across the student's background alternatives. This direct-logit
locality complements the probability-level leakage measurements, which also capture softmax
normalization and shared-parameter effects.

Multiple sampled routes create a separate consistency challenge. Independently clipping each edge
target can violate cycle consistency, so the resulting pairwise demands need not correspond to any
jointly realizable set of token scores. \method{} instead derives every target as the difference
of one bounded teacher potential evaluated at its destination and source. This shared construction
makes all pair targets cycle-consistent and jointly realizable. We adapt the potential range using
the concentration of teacher-deficit demand: concentrated demand identifies clearer destinations
and permits a wider correction budget, whereas diffuse demand yields a more conservative update.
The resulting operator specifies source, destination, and magnitude, reuses the scored teacher
distribution, and adds only $O(k+m)$ routing overhead after top-$k$ extraction.

Experiments across four teacher--student settings and four mathematical-reasoning benchmarks
establish the effectiveness of explicit routing. Relative to sampled reverse-KL OPD, adaptive
\method{} improves the four-benchmark average by $2.19$--$3.61$ points across settings, with a
mean gain of $2.70$ points; paired evaluation intervals over matched problems and samples remain
strictly above zero. It exceeds full-vocabulary reverse KL by $1.24$ points on average, showing
that the gain is not explained merely by access to dense teacher scores, and exceeds fixed-mid
routing by $1.20$ points, isolating the benefit of adaptive budgeting. The improvement is not a
weak-student floor effect: for the strong OpenMath-Nemotron student, the average rises from
$79.73$ under sampled-RKL to $83.34$ under \method{}.

Matched interventions support the proposed transport mechanism rather than an unexamined change
of loss scale. Holding source identity, target magnitude, coupling weight, and nominal budget
fixed, random-background and rank/frequency-matched destinations obtain $62.72$ and $63.72$
Avg@16, whereas teacher-deficit destinations reach $64.43$. On the same frozen diagnostic bank,
\method{} raises routing fidelity from $35.4\%$ for sampled-RKL to $90.4\%$ and reduces
full-vocabulary background leakage from $68.3\%$ to $13.8\%$. Destination identity therefore
changes both the direction of the realized redistribution and downstream reasoning accuracy,
forming the causal link predicted by the transport view.

Further controls validate the remaining design choices and their practical cost. Independent edge
clipping produces $27.6\%$ cycle violations, while the shared potential eliminates these
violations, confirming that its pair targets are jointly realizable. Adaptive budgeting
outperforms fixed-low, fixed-mid, fixed-high, and reversed-concentration controls. Finally, the
default top-$32$ union retains over $97\%$ of both teacher and student mass with $96.38\%$ sign
agreement. It reaches $65.53$ Avg@16 at $2.5\%$ end-to-end overhead, compared with $65.61$ at
$12.5\%$ overhead for full-vocabulary routing. Thus, explicit transport remains both coherent and
sparse rather than trading accuracy for a dense coupling computation.

We make three contributions.
\begin{itemize}
  \item \textbf{A probability-transport view of OPD.} We expose a per-update route
  under-specification in sampled scalar OPD: credit can correct the realized action while leaving
  its redistribution destination to student geometry. We instead formulate categorical
  correction as conserved transport from student-excess sources to teacher-deficit destinations.
  \item \textbf{A sparse, integrable routed objective.} \method{} couples the two marginals,
  realizes each route through pairwise log-odds with zero direct unrelated-logit gradient, and
  derives cycle-consistent targets from a shared bounded potential whose range adapts to teacher
  demand.
  \item \textbf{Effectiveness with causal and practical evidence.} Across four settings,
  \method{} outperforms sampled and full-vocabulary reverse-KL OPD. Matched destination
  interventions and fixed-bank diagnostics link the gains to higher routing fidelity and lower
  leakage, while target-realization and efficiency controls validate the shared potential,
  adaptive budget, and sparse construction.
\end{itemize}

\section{Related Work}

\paragraph{Knowledge distillation and on-policy distillation.}
Knowledge distillation transfers a teacher's predictive distribution to a smaller student
\citep{hinton2015distilling}, with sequence-level and reverse-KL variants extending it to
autoregressive generation \citep{kim2016sequence,gu2024minillm}. Because fixed offline prefixes can
differ from student-generated histories, on-policy distillation (OPD) instead queries the teacher on
states visited by the student \citep{agarwal2024opd}, an idea also adopted in recent open-model
training recipes \citep{qwen3}. Recent theory sharpens when this change of state distribution
matters: online imitation is most useful under student--teacher non-realizability
\citep{zhang2026onlineil}, and noisy expert feedback can create an exponential separation between
offline and online imitation \citep{sriraman2026behavior}. These results explain the value of
student-visited states, but leave open how teacher feedback should produce a categorical update
within each state.

\paragraph{OPD objectives, targets, and stability.}
A central line of work modifies the feedback applied at a visited state. Adaptive target
reformulation constructs an intermediate target to avoid conflicting updates
\citep{jang2026veto}; generalized OPD extrapolates dense teacher feedback and corrects
reference-policy bias \citep{yang2026exopd}; and Uni-OPD combines student-side exploration
balancing with teacher-side reliability calibration \citep{hou2026uniopd}. Complementary diagnoses
identify prefix mismatch and biased top-$k$ reverse-KL estimation and develop corresponding fixes
\citep{zhu2026manyfaces}. More recent methods directly regularize unstable feedback: PowerOPD
replaces the unbounded log-ratio reward with a bounded power transformation
\citep{zhao2026poweropd}, TOP-D dynamically constructs a proximal teacher and reuses data within a
trust region \citep{xie2026topd}, and RG-OPD gates teacher logits using verifier feedback
\citep{akhondzadeh2026rgopd}. These methods determine the target, magnitude, or reliability of
supervision. \method{} is complementary: it determines the vocabulary destination of the
probability removed by a realized correction.

\paragraph{Selecting where and when to distill.}
Another line allocates supervision across the training data or trajectory. PACED weights problems
near the frontier of student competence \citep{xu2026paced}, while PG-OPD uses early
teacher--student overlap to allocate long-rollout budgets to promising trajectories
\citep{zhao2026pgopd}. At finer granularity, TIP selects informative token positions
\citep{xu2026tip}, TRACE routes distillation to critical reasoning spans \citep{wang2026trace}, and
IW-OPD downweights late positions as student prefixes drift from the teacher
\citep{xie2026position}. DEAR further distinguishes uncertain decision tokens from confident but
incorrect evidence tokens \citep{xiao2026dear}. For long-horizon agents, TurnOPD reallocates both
rollout depth and loss budgets across interaction turns \citep{zhou2026turnopd}. All of these methods
answer \emph{which examples, spans, positions, or turns} should receive supervision. At a fixed
selected state, \method{} instead answers \emph{which vocabulary alternative} should receive mass
from a student-excess source. Thus its action-space routing is orthogonal to temporal or data-level
selection.

\paragraph{Understanding OPD's learning dynamics.}
Empirical analyses connect OPD success to teacher--student thinking-pattern compatibility and
progressive alignment on overlapping tokens \citep{li2026rethinkingopd}. Parameter-space studies
provide a complementary view: early OPD updates rapidly stabilize in direction and can be used to
accelerate later training \citep{cai2026foresee}, while broader checkpoint analyses find that dense
teacher supervision nevertheless produces small, coordinate-sparse, and spectrally concentrated
updates \citep{yu2026geometry}. These findings characterize \emph{when} OPD succeeds and
\emph{how} its changes accumulate. Our analysis operates one level earlier, exposing the
per-update softmax redistribution that generates those dynamics and replacing its implicit
destination rule with explicit excess-to-deficit routes.

\begin{figure*}[t]
  \centering
  \includegraphics[width=\textwidth]{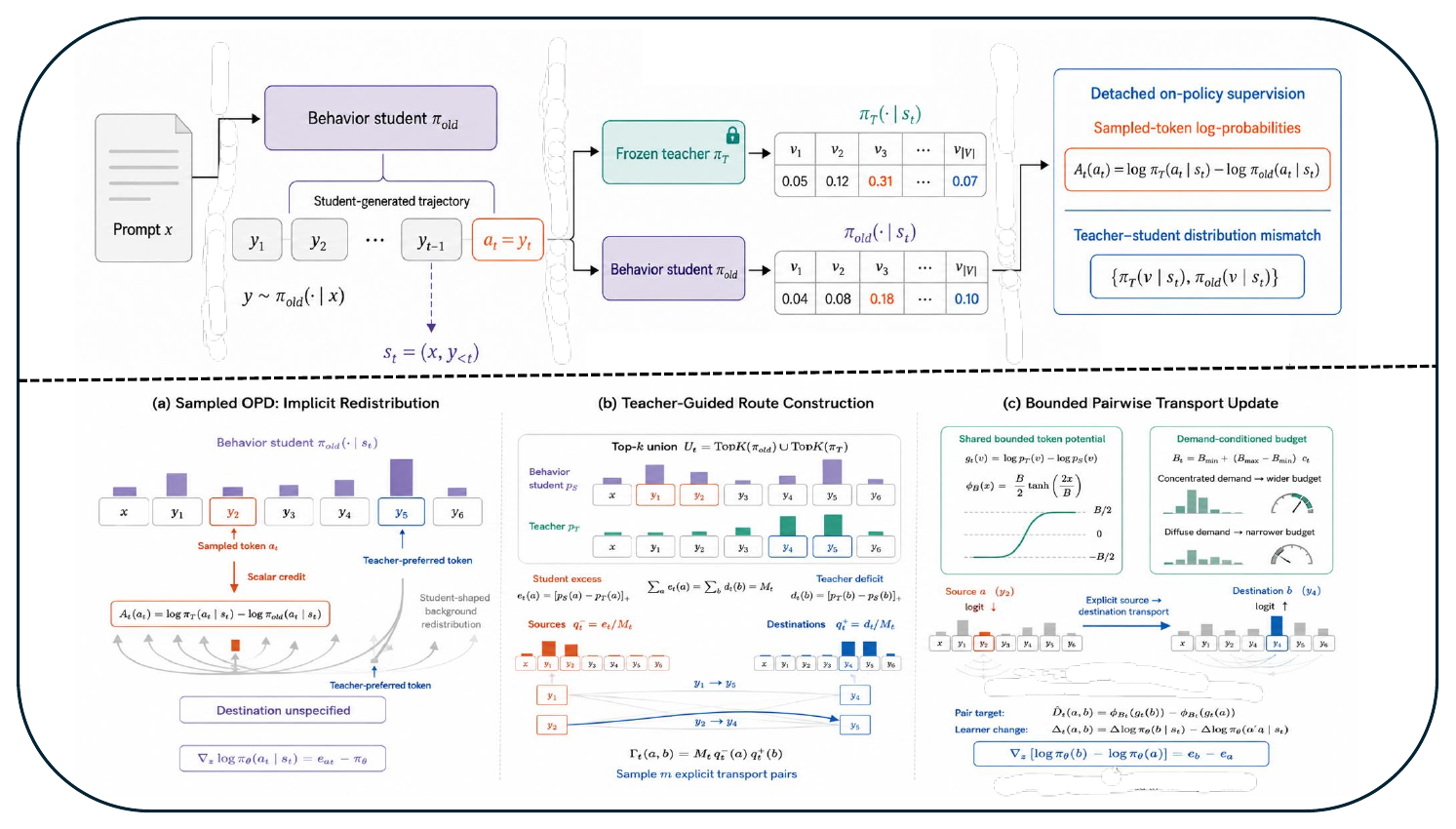}
  \caption{\textbf{Overview of \method{}.} Sampled OPD (a) assigns scalar credit to one token and
  leaves redistribution to student softmax geometry. \method{} (b) decomposes teacher--student
  disagreement into excess sources and deficit destinations and samples explicit routes.
  A bounded, demand-conditioned token potential (c) supplies integrable pairwise log-odds targets,
  producing zero direct gradient on unrelated logits.}
  \label{fig:routeopd-overview}
\end{figure*}

\input{sections/main_results_table}

\section{Method: \method}

\method{} converts teacher--student disagreement at a student-visited state into explicit
source--destination updates with teacher-directed destinations, direct logit locality, and jointly
realizable targets.

\subsection{From scalar feedback to explicit routes}
\paragraph{Implicit redistribution in sampled OPD.}
For a prompt $x$, the behavior student samples $y\sim\pi_{\mathrm{old}}(\cdot\mid x)$ and visits
$s_t=(x,y_{<t})$, with $a_t=y_t$. The sampled reverse-KL baseline uses the detached advantage
\begin{equation}
  A_t(a_t)=\operatorname{sg}\!\left[
  \log\pi_T(a_t\mid s_t)
  -\log\pi_{\mathrm{old}}(a_t\mid s_t)\right],
  \label{eq:sampled-opd-advantage}
\end{equation}
where $\operatorname{sg}$ is stop-gradient. With importance ratio
$\rho_t=\pi_\theta(a_t\mid s_t)/\pi_{\mathrm{old}}(a_t\mid s_t)$, the local loss is
$-\rho_tA_t$. At the start of an update, a negative advantage produces
\begin{equation}
  -\nabla_z\ell_{\mathrm{sOPD}}^{(t)}
  =|A_t(a_t)|
  \left[\pi_{\mathrm{old}}(\cdot\mid s_t)-\mathbf e_{a_t}\right].
  \label{eq:implicit-redistribution}
\end{equation}
Thus the sampled token is suppressed while student geometry distributes the released mass over
non-source tokens: the teacher identifies the correction sign, not its recipient
(Appendix~\ref{app:scalar-redistribution}).

\paragraph{Excess and deficit.}
At rollout/scoring time, form a sparse shared support and separately renormalize each model on it:
\begin{equation}
  \mathcal U_t
  =\operatorname{TopK}(\pi_{\mathrm{old}}(\cdot\mid s_t))
  \cup\operatorname{TopK}(\pi_T(\cdot\mid s_t)).
  \label{eq:topk-union}
\end{equation}
\begin{equation}
  p_X(a)=
  \frac{\pi_X(a\mid s_t)}
  {\sum_{v\in\mathcal U_t}\pi_X(v\mid s_t)},\quad
  X\in\{S,T\},\ a\in\mathcal U_t,
  \label{eq:union-distributions}
\end{equation}
where $\pi_S=\pi_{\mathrm{old}}$. All routing statistics are detached within an update. Define
\begin{equation}
  e_t(a)=[p_S(a)-p_T(a)]_+,\qquad
  d_t(a)=[p_T(a)-p_S(a)]_+.
  \label{eq:excess-deficit}
\end{equation}
Their common mass is
\begin{equation}
  M_t
  =\sum_{a\in\mathcal U_t}e_t(a)
  =\sum_{a\in\mathcal U_t}d_t(a)
  =\operatorname{TV}(p_S,p_T),
  \label{eq:mismatch-mass}
\end{equation}
so $e_t$ identifies sources, $d_t$ identifies destinations, and $M_t$ is the available transport
mass. States with $M_t\leq\delta_M$ receive zero route loss.

\subsection{Sparse excess-to-deficit routing}
For $M_t>\delta_M$, normalize both marginals and use their independent product coupling:
\begin{equation}
  \begin{aligned}
  q_t^-(a)&=\frac{e_t(a)}{M_t},\qquad
  q_t^+(b)=\frac{d_t(b)}{M_t},\\
  \Gamma_t(a,b)&=M_tq_t^-(a)q_t^+(b).
  \end{aligned}
  \label{eq:transport-marginals}
\end{equation}
This coupling preserves both marginals; sampling $m$ independent routes
$a_i\sim q_t^-,b_i\sim q_t^+$ avoids the $k^2$ table. Appendix~\ref{app:coupling-details} derives
the marginal identities and unbiased estimator.

\subsection{Directed and integrable pairwise targets}
\paragraph{Direct locality.}
For route $(a,b)$, define the pairwise log-odds change
\begin{equation}
  \begin{aligned}
  \Delta_t(a,b)
  ={}&\bigl[\log\pi_\theta(b\mid s_t)
       -\log\pi_\theta(a\mid s_t)\bigr]\\
     &-\bigl[\log\pi_{\mathrm{old}}(b\mid s_t)
       -\log\pi_{\mathrm{old}}(a\mid s_t)\bigr].
  \end{aligned}
  \label{eq:pairwise-logodds-change}
\end{equation}
Its logit gradient is
\begin{equation}
  \nabla_z\!\left(\log\pi_\theta(b)-\log\pi_\theta(a)\right)
  =\mathbf e_b-\mathbf e_a.
  \label{eq:pairwise-routing-gradient}
\end{equation}
Increasing $\Delta_t$ directly raises only the destination logit and lowers only the source logit;
this establishes direct-logit locality, while Appendix~\ref{app:probability-locality}
characterizes the induced probability changes under softmax normalization.

\paragraph{Teacher target and adaptive budget.}
Let $\ell_{\delta_p}(x)=\log(\max\{x,\delta_p\})$ and define the teacher token potential
\begin{equation}
  g_t(v)=\ell_{\delta_p}(p_T(v))-\ell_{\delta_p}(p_S(v)).
  \label{eq:teacher-token-potential}
\end{equation}
For the nonempty deficit set $\mathcal D_t=\{b\in\mathcal U_t:d_t(b)>0\}$ with
$n_t=|\mathcal D_t|$, measure teacher-demand concentration by the normalized Herfindahl index
\begin{equation}
  c_t=
  \begin{cases}
    1, & n_t=1,\\[2pt]
    \displaystyle
    \frac{n_t\sum_{b\in\mathcal D_t}q_t^+(b)^2-1}
         {n_t-1}, & n_t>1,
  \end{cases}
  \label{eq:demand-concentration}
\end{equation}
where $c_t\in[0,1]$. For $0<B_{\min}\leq B_{\max}$, set the target budget to
\begin{equation}
  B_t=B_{\min}+(B_{\max}-B_{\min})c_t.
  \label{eq:transport-budget}
\end{equation}
Rather than independently clipping pair targets, bound one shared token potential:
\begin{equation}
  \begin{aligned}
  \phi_B(x)&=\frac{B}{2}\tanh\!\left(\frac{2x}{B}\right),\\
  \widehat D_t(a,b)
  &=\phi_{B_t}(g_t(b))-\phi_{B_t}(g_t(a)).
  \end{aligned}
  \label{eq:bounded-potential-target}
\end{equation}
Monotonicity gives $\widehat D_t(a,b)\in[0,B_t)$, and the shared-potential shift
$z_v\leftarrow z_v+\phi_{B_t}(g_t(v))$ jointly realizes every pair target.
Appendix~\ref{app:target-properties} proves positivity, boundedness, and integrability and contrasts
independent edge clipping.

\subsection{Objective and implementation}
With Huber transition $\kappa>0$,
\begin{equation}
  h_\kappa(r)=
  \begin{cases}
    \frac{1}{2}r^2, & |r|\leq\kappa,\\[2pt]
    \kappa\left(|r|-\frac{1}{2}\kappa\right), & |r|>\kappa.
  \end{cases}
  \label{eq:huber-loss}
\end{equation}
the \method{} token loss is
\begin{equation}
  \mathcal L_{\mathrm{RouteOPD}}^{(t)}
  =M_t\Ex_{(a,b)\sim q_t^-\times q_t^+}
  \!\left[h_\kappa\!\left(r_t(a,b)\right)\right].
  \label{eq:routeopd-objective}
\end{equation}
\method{} replaces sampled reverse-KL; pairs and routing statistics are frozen within each update,
so only current-student terms in $\Delta_t$ are differentiable. Let $\mu_t$ mark valid generated
response positions.

\begin{algorithm}[t]
\caption{\method{} at one student-visited state}
\label{alg:routeopd}
\textbf{Input}: response mask $\mu_t$; teacher/old-student top-$k$ IDs;\\
detached cross-gathered $\log\pi_T,\log\pi_{\mathrm{old}}$ on their union;\\
differentiable current $\log\pi_\theta$ for sampled-ID gathers\\
\textbf{Output}: one \method{} token loss
\begin{algorithmic}[1]
\STATE form $\mathcal U_t$ from IDs; build detached $p_S,p_T$
\STATE compute $e_t,d_t,M_t$ by Eqs.~\eqref{eq:excess-deficit}--\eqref{eq:mismatch-mass}
\STATE \textbf{if} $\mu_t=0$ or $M_t\leq\delta_M$ \textbf{return} $0$
\STATE form $q_t^-,q_t^+$ and sample $m$ independent pairs
\STATE compute $g_t,c_t,B_t$ by Eqs.~\eqref{eq:teacher-token-potential}, \eqref{eq:demand-concentration}, and \eqref{eq:transport-budget}
\STATE bound token potentials and form $\widehat D_t$ by Eq.~\eqref{eq:bounded-potential-target}
\STATE compute $\Delta_t(a_i,b_i)$ from current/old log-probabilities
\STATE \textbf{return} $\frac{M_t}{m}\sum_{i=1}^m h_\kappa(\Delta_t(a_i,b_i)-\widehat D_t(a_i,b_i))$
\end{algorithmic}
\end{algorithm}

After top-$k$ scoring, routing costs $O(k+m)$ per valid token and reuses teacher logits without an
extra forward pass. Appendix~\ref{app:implementation-details} specifies cross-gathering, validity,
numerical conventions, defaults, and independent-edge clipping.

\section{Experiments}
\subsection{Experimental Setup}
\paragraph{Setup.}
We evaluate four teacher--student pairs spanning Qwen3, DeepSeek, and Nemotron lineages on
MATH500 \citep{hendrycks2021math,lightman2024verify,math500dataset}, AMC23 \citep{amc2023},
AIME24 \citep{aime2024}, and AIME25 \citep{aime2025} using Avg@16 exact-answer accuracy.
All trainable arms share DAPO-Math-17K \citep{yu2025dapo,dapoMath17k}, rollout seeds, teacher calls, and update
budgets; token-level comparisons require exact tokenizer compatibility. We compare sampled-RKL,
full-vocabulary KL, fixed-mid routing, and adaptive \method{}. Unless varied, $k=32$, $m=2$,
$B_{\min}=\log1.2$, and $B_{\max}=\log1.5$. Full training, decoding, statistical, and
decontamination details are deferred to the supplementary material.

\subsection{How Well Does \method{} Perform?}

\noindent\textbf{Finding 1: \method{} improves every teacher--student setting.}
Relative to sampled-RKL, \method{} improves the four-benchmark average by $2.19$--$3.61$ points
($2.70$ on average). It also exceeds full-vocabulary RKL by $1.13$--$1.44$ points and fixed-mid routing
by $1.10$--$1.43$ points. The gains persist from a weak Qwen3 base student to the strong
OpenMath-Nemotron student, indicating that explicit routing is not tied to one initialization or
teacher scale.
Matched problem-and-sample bootstrap intervals are also uniformly positive: the macro improvements
are $+2.70\,[2.28,3.13]$ over sampled-RKL, $+1.24\,[0.90,1.58]$ over full-vocabulary RKL, and
$+1.20\,[0.89,1.51]$ over fixed-mid routing, confirming the gains across matched evaluation
resamples.

\subsection{Why Does Explicit Routing Help?}
We now test the three decisions that define explicit routing: \emph{where} probability should go,
\emph{how} the requested changes should be realized jointly, and \emph{how much} transport each
state should receive.

\begin{figure}[t]
  \centering
  \includegraphics[width=\columnwidth]{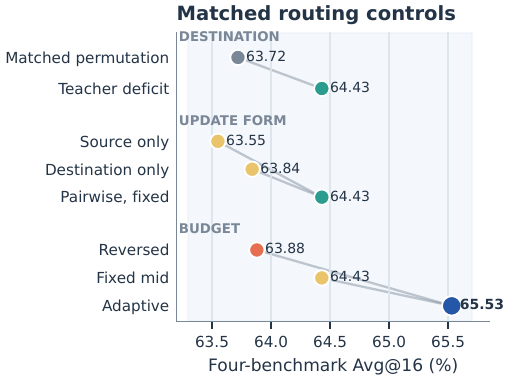}
  \caption{\textbf{Matched routing controls.}
  Each block varies one design family while retaining the remaining fixed-route construction;
  the points are separate contrasts rather than a sequential path from sampled-RKL.}
  \label{fig:component-path}
\end{figure}

\noindent\textbf{Finding 2: Destination identity and a two-sided update both matter.}
Figure~\ref{fig:component-path} first isolates where removed probability is assigned. A
rank/frequency-matched destination permutation reaches $63.72$, whereas teacher-deficit
destinations reach $64.43$ under the same fixed-mid routed objective; random-background
destinations fall further to $62.72$. The update-form block then compares the same routed
construction with one side disabled: source-only and destination-only updates reach $63.55$ and
$63.84$, respectively, compared with $64.43$ for the pairwise update. Thus the gain is not
explained by selecting unusual tokens or updating only one endpoint: the destination must follow
teacher deficit, and the source and destination must be corrected together.

\noindent\textbf{Finding 3: Shared potentials make route targets jointly realizable.}
Independent clipping treats each sampled edge as an unrelated target and produces $27.6\%$
four-cycle violations. Deriving all pair targets from one bounded token potential eliminates these
violations, reduces target MAE from $0.091$ to $0.067$, and raises bounded target completion from
$76.4\%$ to $83.2\%$ under the fixed-mid budget. With adaptive budgeting, MAE further falls to
$0.052$ and completion reaches $88.1\%$. These diagnostics isolate the consistency benefit of the
shared potential, while the following budget comparison isolates magnitude adaptation.

\noindent\textbf{Finding 4: Adaptive budgeting helps when teacher demand is concentrated.}
The aggregate comparison in Figure~\ref{fig:component-path} favors adaptive over fixed-mid
budgeting, $65.53$ versus $64.43$, while reversing the concentration rule reaches only $63.88$.
More importantly, the adaptive-minus-fixed difference grows with teacher-deficit concentration:
$-0.10$, $+0.40$, $+1.30$, and $+2.70$ points across the four concentration bins. Adaptive
budgeting therefore does not merely enlarge every update; it allocates additional transport in
the states where teacher demand is most concentrated.

\begin{figure}[!ht]
  \centering
  \includegraphics[width=\columnwidth]{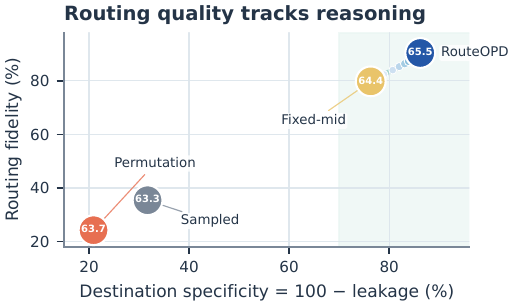}
  \caption{\textbf{Mechanism phase portrait.}
  Large markers are matched method variants with Avg@16 printed inside; small blue markers trace
  adaptive \method{} checkpoints. Desired transport lies toward the upper right.}
  \label{fig:transport-path}
\end{figure}

\noindent\textbf{Finding 5: Better task performance accompanies a stronger routing signature.}
Figure~\ref{fig:transport-path} places the same interventions in fidelity--specificity space.
Sampled-RKL has $35.4\%$ routing fidelity and $68.3\%$ full-vocabulary background leakage; the
matched destination permutation degrades these to $24.2\%$ and $79.1\%$. Teacher-deficit
destinations with a fixed-mid budget instead reach $79.8\%$ fidelity and $23.7\%$ leakage, and
adaptive \method{} reaches $90.4\%$ and $13.8\%$. Along adaptive \method{} checkpoints, this
routing signature strengthens while AIME24/25 Avg@16 rises from $27.19$ to $43.65$. Together with
the matched destination intervention, the trajectory is consistent with the proposed mechanism;
the checkpoint correlation alone is not treated as causal identification.

\paragraph{The effect is broad across states.}
The fixed 1,024-state bank confirms that the mean effect is broad: adaptive \method{} reaches
$92.7\%$ median fidelity and $76.5\%$ at the tenth percentile, versus $33.1\%$ median fidelity for
sampled-RKL, while median full-vocabulary leakage falls from $70.1\%$ to $11.2\%$.

\subsection{Training Dynamics and Efficiency}

\begin{figure}[t]
  \centering
  \includegraphics[width=\columnwidth]{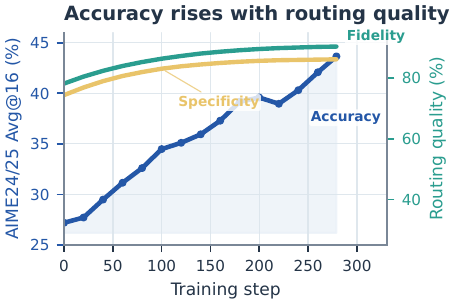}
  \caption{\textbf{Training evidence on JustRL--DeepSeek.}
  Accuracy, routing fidelity, and destination specificity ($100-$background leakage) improve
  together.}
  \label{fig:training-evidence}
\end{figure}

\noindent\textbf{Finding 6: The routing signature emerges throughout training.}
Figure~\ref{fig:training-evidence} shows a largely monotone accuracy rise as routing fidelity and
destination specificity improve and then saturate, rather than an isolated favorable checkpoint.
The appendix compares mismatch, fidelity, and leakage trajectories with sampled-RKL and fixed-mid
routing.

\paragraph{Sparse support and pair sampling.}
The sparse construction retains $97.79\%$ of student mass and $97.13\%$ of teacher mass at $k=32$.
Its $65.53$ Avg@16 requires only $2.5\%$ end-to-end overhead, whereas full-vocabulary routing
reaches $65.61$ at $12.5\%$ overhead. For route sampling, $m=2$ is best in this setting:
$m=1$ lowers target coverage from $76.8\%$ to $58.4\%$, whereas $m=4$ and $m=8$ reduce route-loss
variance but also reduce Avg@16 from $65.53$ to $62.48$ and $61.12$. The resulting optimum at
$m=2$ balances estimator variance with selective target coverage.

\begin{figure}[t]
  \centering
  \includegraphics[width=\columnwidth]{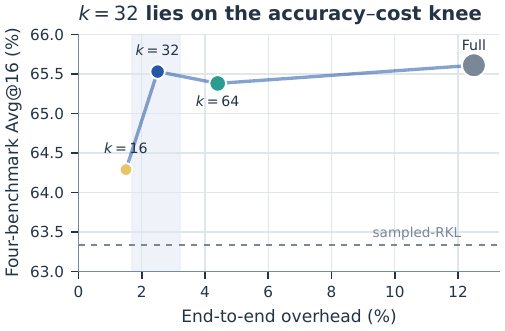}
  \caption{\textbf{Accuracy--cost trade-off.}
  Marker area encodes peak memory. The sparse $k=32$ default lies at the Pareto knee and recovers
  nearly all full-vocabulary accuracy at one fifth of its end-to-end overhead.}
  \label{fig:efficiency-pareto-main}
\end{figure}

\noindent\textbf{Finding 7: Sparse routing captures the practical Pareto knee.}
Figure~\ref{fig:efficiency-pareto-main} shows the largest accuracy gain from $k=16$ to $32$;
$k=64$ and full-vocabulary routing add substantial cost for little further accuracy on this system.

\subsection{Behavioral Outcome Analysis}

\begin{figure}[t]
  \centering
  \includegraphics[width=\columnwidth]{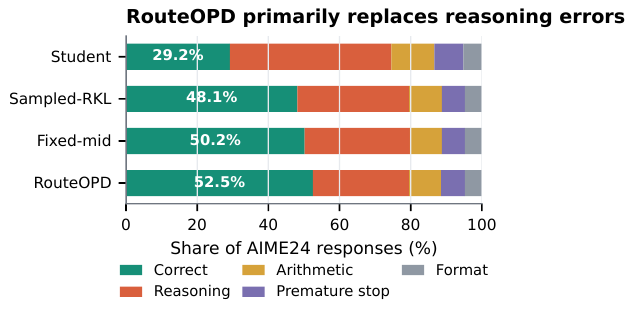}
  \caption{\textbf{AIME24 response composition.}
  \method{} gains primarily replace reasoning errors, not formatting or truncation failures.}
  \label{fig:error-composition-main}
\end{figure}

\noindent\textbf{Finding 8: On AIME24, explicit routing primarily repairs reasoning errors.}
Figure~\ref{fig:error-composition-main} shows that the correct-response share rises from $48.1\%$
under sampled-RKL and $50.2\%$ under fixed-mid routing to $52.5\%$ under \method{}. Among 480
matched responses, \method{} converts 42 sampled-RKL failures to correct solutions while losing
21 previously correct solutions; against fixed-mid routing, the corresponding counts are 27 and
16.
Of the 42 sampled-RKL failures repaired by \method{}, 34 are reasoning errors, compared with five
arithmetic errors, two premature stops, and one formatting error. Format validity remains
$95.21\%$, while the clipping rate decreases from $4.47\%$ to $2.34\%$. Two method-blinded passes
of a fixed LLM judge reach $92.4\%$ raw agreement and Cohen's $\kappa=0.864$. These results
show that recovered reasoning is the dominant behavioral change, accompanied by stable format
validity and a lower clipping rate.

\section{Conclusion}
Sampled OPD leaves redistributed probability mass without an explicit destination. \method{}
resolves this ambiguity through teacher-guided transport from student-excess sources to
teacher-deficit destinations, using cycle-consistent targets and demand-adaptive magnitudes.
Across four teacher--student settings and four reasoning benchmarks, \method{} consistently
outperforms sampled reverse-KL OPD, full-vocabulary reverse KL, and fixed-mid routing. Controlled
diagnostics attribute these gains to more faithful routing with less leakage, while sparse routing
retains nearly full-vocabulary accuracy at substantially lower cost.

\bibliography{opd}

\clearpage
\appendix
\input{sections/experiment_appendix}

\end{document}

%% file: sections/main_results_table.tex
\begin{table*}[t]
\centering
\footnotesize
\setlength{\tabcolsep}{4.0pt}
\renewcommand{\arraystretch}{1.02}
\caption{\textbf{Mathematical-reasoning accuracy across four teacher--student settings.}
Entries are Avg@16 accuracy (\%); Avg. is the arithmetic mean of the four benchmarks. Fixed-mid
routing uses $B=\log1.35$. \textbf{Bold} marks the best student method in each block.}
\label{tab:main-wide}
\begin{adjustbox}{max width=\textwidth}
\begin{tabular*}{\textwidth}{@{\extracolsep{\fill}}lrrrrr@{}}
\toprule
Method & MATH500 & AMC23 & AIME24 & AIME25 & Avg. \\
\midrule
\multicolumn{6}{l}{\textit{Qwen3-4B-Base-GRPO $\to$ Qwen3-1.7B-Base}} \\
Student (untrained; no OPD) & 45.31 & 11.67 & 1.46 & 1.04 & 14.87 \\
Sampled-RKL OPD & 70.18 & 37.65 & 11.25 & 6.67 & 31.44 \\
Full-vocabulary RKL & 70.86 & 39.16 & 12.08 & 7.29 & 32.35 \\
Routed OPD, fixed-mid $B=\log1.35$ & 70.63 & 39.46 & 11.88 & \textbf{7.50} & 32.36 \\
\textbf{\method{} (ours)} & \textbf{71.48} & \textbf{42.85} & \textbf{13.33} &
  \textbf{7.50} & \cellcolor{avgshade!55}\textbf{33.79} \\
\midrule
\multicolumn{6}{l}{\textit{JustRL-DeepSeek-1.5B $\to$ DeepSeek-R1-Distill-Qwen-1.5B}} \\
Student (untrained; no OPD) & 84.90 & 69.05 & 29.17 & 25.21 & 52.08 \\
Sampled-RKL OPD & 86.63 & 85.47 & 48.13 & 33.13 & 63.34 \\
Full-vocabulary RKL & 87.51 & 86.14 & 49.38 & 34.58 & 64.40 \\
Routed OPD, fixed-mid $B=\log1.35$ & 87.31 & 86.45 & 50.21 & 33.75 & 64.43 \\
\textbf{\method{} (ours)} & \textbf{88.08} & \textbf{86.75} & \textbf{52.50} &
  \textbf{34.79} & \cellcolor{avgshade!55}\textbf{65.53} \\
\midrule
\multicolumn{6}{l}{\textit{JustRL-Nemotron-1.5B $\to$ OpenMath-Nemotron-1.5B}} \\
Student (untrained; no OPD) & 92.40 & 88.10 & 53.13 & 48.13 & 70.44 \\
Sampled-RKL OPD & 93.21 & \textbf{96.54} & 68.96 & 60.21 & 79.73 \\
Full-vocabulary RKL & 93.13 & 96.23 & 73.54 & 65.63 & 82.13 \\
Routed OPD, fixed-mid $B=\log1.35$ & 93.28 & 96.08 & 74.38 & 65.21 & 82.24 \\
\textbf{\method{} (ours)} & \textbf{93.53} & 96.08 & \textbf{75.83} &
  \textbf{67.92} & \cellcolor{avgshade!55}\textbf{83.34} \\
\midrule
\multicolumn{6}{l}{\textit{DeepSeek-R1-Distill-Qwen-7B $\to$
DeepSeek-R1-Distill-Qwen-1.5B}} \\
Student (untrained; no OPD) & 84.90 & 69.05 & 29.17 & 25.21 & 52.08 \\
Sampled-RKL OPD & 88.24 & 70.93 & 30.63 & 24.58 & 53.59 \\
Full-vocabulary RKL & 88.63 & 72.89 & 32.71 & 26.04 & 55.07 \\
Routed OPD, fixed-mid $B=\log1.35$ & 88.81 & 72.67 & 32.50 & 26.25 & 55.06 \\
\textbf{\method{} (ours)} & \textbf{88.93} & \textbf{74.62} & \textbf{33.96} &
  \textbf{27.50} & \cellcolor{avgshade!55}\textbf{56.25} \\
\bottomrule
\end{tabular*}
\end{adjustbox}
\end{table*}

%% file: sections/experiment_appendix.tex
\section{Proofs and Additional Method Details}
\label{app:method-details}

\subsection{Scalar OPD Redistribution}
\label{app:scalar-redistribution}

\begin{proposition}[Implicit redistribution under sampled OPD]
\label{prop:implicit-redistribution}
Fix a visited state $s_t$ and suppose the learner is evaluated at the behavior policy,
$\pi_\theta=\pi_{\mathrm{old}}$. If the sampled action $a_t$ has detached advantage
$A_t(a_t)<0$, then the negative gradient of the sampled-RKL loss is
\[
  -\nabla_z\ell_{\mathrm{sOPD}}^{(t)}
  =|A_t(a_t)|\bigl[\pi_{\mathrm{old}}-\mathbf e_{a_t}\bigr].
\]
Consequently, its source-logit component is negative and every non-source component is positive.
Their components sum to zero.
\end{proposition}

\begin{proof}
Let $z$ be the learner logits at one visited state and abbreviate
$\pi=\pi_\theta(\cdot\mid s_t)$. The log-softmax identity gives
\begin{equation}
  \begin{aligned}
  \log\pi(a)&=z_a-\log\sum_v\exp(z_v),\\
  \nabla_z\log\pi(a)&=\mathbf e_a-\pi,
  \end{aligned}
  \label{eq:single-token-logit-gradient}
\end{equation}
where $\mathbf e_a$ is the one-hot vector for token $a$. At the start of an update,
$\pi_\theta=\pi_{\mathrm{old}}$ and $\rho_t=1$. Substituting
Eq.~\eqref{eq:single-token-logit-gradient} into the sampled-RKL surrogate yields
Eq.~\eqref{eq:implicit-redistribution}. A direct descent step for $A_t(a_t)<0$ lowers the sampled
source component in proportion to $1-\pi_{\mathrm{old}}(a_t\mid s_t)$ and raises each non-source
component $v\neq a_t$ in proportion to $\pi_{\mathrm{old}}(v\mid s_t)$. Ratio clipping can limit
the step magnitude, but it does not change this single-action score direction.
\end{proof}

\subsection{Coupling Identities and Monte Carlo Estimator}
\label{app:coupling-details}

\begin{proposition}[Exact marginals, maximum entropy, and unbiased estimation]
\label{prop:coupling-estimator}
Let $p_S$ and $p_T$ be probability distributions on the same finite support $\mathcal U_t$, and
let $e_t,d_t,M_t,q_t^-,q_t^+$ be defined as in
Eqs.~\eqref{eq:mismatch-mass}--\eqref{eq:transport-marginals}. Then
$M_t=\frac12\lVert p_S-p_T\rVert_1$. If $M_t>0$, the product coupling
$\Gamma_t(a,b)=M_tq_t^-(a)q_t^+(b)$ has marginals $e_t$ and $d_t$ and, after division by $M_t$,
is the maximum-entropy joint distribution with marginals $q_t^-$ and $q_t^+$. For any finite
pair cost $C_t$, the $m$-pair estimator in Eq.~\eqref{eq:transport-monte-carlo} is unbiased.
\end{proposition}

\begin{proof}
Because $\sum_v(p_S(v)-p_T(v))=0$, the total positive and negative parts are equal. Therefore
\begin{equation}
  \begin{aligned}
  \sum_v[p_S(v)-p_T(v)]_+
    &=\sum_v[p_T(v)-p_S(v)]_+\\
    &=\tfrac12\sum_v|p_S(v)-p_T(v)|=M_t.
  \end{aligned}
\end{equation}
The product coupling in Eq.~\eqref{eq:transport-marginals} is the maximum-entropy coupling of the
fixed excess and deficit marginals. It preserves both sides exactly:
\begin{equation}
  \begin{aligned}
  \sum_{b\in\mathcal U_t}\Gamma_t(a,b)
    &=M_tq_t^-(a)=e_t(a),\\
  \sum_{a\in\mathcal U_t}\Gamma_t(a,b)
    &=M_tq_t^+(b)=d_t(b).
  \end{aligned}
  \label{eq:transport-exact-marginals}
\end{equation}
For any pair cost $C_t$ and positive integer $m$, its coupling objective and sampled estimator are
\begin{equation}
  \begin{aligned}
  \sum_{a,b}\Gamma_t(a,b)C_t(a,b)
    &=M_t\Ex_{q_t^-\times q_t^+}[C_t],\\
  \widehat{\mathcal C}_t
    &=\frac{M_t}{m}\sum_{i=1}^{m}C_t(a_i,b_i),
  \end{aligned}
  \label{eq:transport-monte-carlo}
\end{equation}
where $a_i\sim q_t^-$ and $b_i\sim q_t^+$ independently. Hence
$\Ex[\widehat{\mathcal C}_t]$ equals the exact coupling objective without materializing the
$k^2$ source--destination table.

For completeness, let $(A,B)$ have any normalized joint distribution with these marginals.
Its entropy satisfies
$H(A,B)=H(A)+H(B)-I(A;B)\leq H(q_t^-)+H(q_t^+)$ because mutual information is nonnegative.
Equality holds for the independent joint $q_t^-\times q_t^+$, proving the maximum-entropy claim.
Finally, linearity of expectation applied to the $m$ independent costs yields
\[
  \Ex\!\left[\frac{M_t}{m}\sum_{i=1}^m C_t(a_i,b_i)\right]
  =M_t\Ex_{q_t^-\times q_t^+}[C_t],
\]
which proves unbiasedness.
\end{proof}

\subsection{Direct-Logit versus Probability Locality}
\label{app:probability-locality}

\begin{proposition}[Direct-logit locality of a routed update]
\label{prop:pairwise-locality}
For any distinct vocabulary items $a$ and $b$ under a softmax policy,
\[
  \nabla_z\!\left[\log\pi_\theta(b\mid s_t)-\log\pi_\theta(a\mid s_t)\right]
  =\mathbf e_b-\mathbf e_a.
\]
Thus the pairwise objective has no direct gradient on unrelated logits. Nevertheless, an
infinitesimal direct-logit step can change unrelated probabilities through softmax normalization.
\end{proposition}

\begin{proof}
Since $\log\pi_\theta(v\mid s_t)=z_v-\log\sum_j e^{z_j}$, subtracting the two log-probabilities
cancels the log-partition term and leaves $z_b-z_a$. Differentiation gives
$\mathbf e_b-\mathbf e_a$.

Equation~\eqref{eq:pairwise-routing-gradient} characterizes direct-logit locality. For an
infinitesimal direct-logit step
$dz=\eta(\mathbf e_b-\mathbf e_a)$ and any $v\notin\{a,b\}$, the softmax differential is
\begin{equation}
  \begin{aligned}
  d\pi(v)
  &=\pi(v)\!\left(dz_v-\sum_j\pi(j)dz_j\right)\\
  &=-\eta\pi(v)\bigl(\pi(b)-\pi(a)\bigr).
  \end{aligned}
  \label{eq:unrelated-probability-change}
\end{equation}
Thus normalization can change unrelated probabilities even when their logits receive no direct
displacement. A shared-parameter optimizer step can additionally move their logits.
\end{proof}

\subsection{Target Positivity, Boundedness, and Integrability}
\label{app:target-properties}

\begin{proposition}[Valid and jointly integrable routed targets]
\label{prop:integrable-targets}
Let $a$ be a student-excess token and $b$ a teacher-deficit token. For $B_t>0$, the shared-potential
target in Eq.~\eqref{eq:bounded-potential-target} satisfies
$0\leq\widehat D_t(a,b)<B_t$. It is strictly positive unless probability-floor saturation makes
the relevant teacher and student log-potentials equal. Moreover, every routed edge target is
jointly integrable: a single token potential realizes all pairwise targets simultaneously.
\end{proposition}

\begin{proof}
Before bounding, the teacher-implied pairwise target is
\begin{equation}
  D_t^\star(a,b)=g_t(b)-g_t(a).
  \label{eq:teacher-pairwise-target}
\end{equation}
For an excess source and deficit destination, $p_S(a)>p_T(a)$ and $p_T(b)>p_S(b)$. Since
$\ell_{\delta_p}$ is nondecreasing,
\begin{equation}
  \begin{aligned}
  D_t^\star(a,b)
  ={}&\ell_{\delta_p}(p_T(b))-\ell_{\delta_p}(p_S(b))\\
     &+\ell_{\delta_p}(p_S(a))-\ell_{\delta_p}(p_T(a))
     \geq0.
  \end{aligned}
  \label{eq:positive-pairwise-target}
\end{equation}
The inequality is strict unless probability-floor saturation collapses both ordered comparisons.
The bounded map in Eq.~\eqref{eq:bounded-potential-target} is locally the identity because
$\phi'_B(0)=1$ and saturates token potentials in $(-B/2,B/2)$. Consequently
$\widehat D_t(a,b)\in[0,B_t)$ for every excess-to-deficit route.

Define $h_t(v)=\phi_{B_t}(g_t(v))$. Every target is $h_t(b)-h_t(a)$, so the direct shift
$z_v\leftarrow z_v+h_t(v)$ changes every pairwise log-odds by exactly $\widehat D_t(a,b)$.
This proves joint integrability.

Independent edge clipping does not share this guarantee. For example, take two sources and two
destinations with potentials
$g(a_1)=-2B$, $g(a_2)=-0.4B$, $g(b_1)=0.4B$, and $g(b_2)=2B$.
Clipping each positive difference at $B$ gives
$D_{11}=D_{12}=D_{22}=B$ and $D_{21}=0.8B$. These violate the necessary four-cycle identity
$D_{11}+D_{22}=D_{12}+D_{21}$, because $2B\neq1.8B$; hence no shared token potential realizes
all four clipped targets.
\end{proof}

Teacher-demand concentration controls $B_t$ as described in
Eq.~\eqref{eq:transport-budget}; Section~\ref{app:complete-results} evaluates this design against
fixed and reversed-concentration budgets.

\subsection{Implementation and Numerical Conventions}
\label{app:implementation-details}

A response state is valid exactly when it is a generated, unpadded response position, its
teacher/behavior-student scores are finite and token-aligned, and $\mathcal U_t$ is nonempty:
\begin{equation}
  \mu_t=\ind\{\text{response state $t$ is valid}\}.
  \label{eq:response-validity-mask}
\end{equation}
The probability floor $\delta_p$, mismatch skip tolerance $\delta_M$, diagnostic denominator
tolerance $\delta_{\rm diag}$, and substantive target threshold $\delta_{\rm target}$ have distinct
roles and are never reused for one another.

At scoring time, each model retains its top-$k$ IDs and cross-gathers detached log-probabilities at
the other model's IDs; the deduplicated values define $\mathcal U_t,p_S,p_T$. During the learner
update, differentiable current log-probabilities are gathered only at sampled pair IDs.
Vocabulary-wide scoring and top-$k$ extraction are excluded from the incremental route-overhead
count. After those operations, the overhead is $O(k+m)$ per valid token: $O(k)$ for detached
routing statistics and $O(m)$ for pair sampling and learner gathers. The method reuses already
scored teacher logits and introduces no extra teacher forward pass.

The defaults are $k=32$, $m=2$, $B_{\min}=\log1.2$, $B_{\max}=\log1.5$, $\kappa=1.0$,
$\delta_p=10^{-8}$, $\delta_M=10^{-6}$, $\delta_{\rm diag}=10^{-12}$, and
$\delta_{\rm target}=0.01$ log-odds. All routing statistics are detached per update, and invalid
response positions receive zero loss. Independent edge clipping is used only as an ablation.

\section{Full Experimental Protocol}
\label{app:protocol}

\subsection{Teacher--Student Settings and Tokenizer Compatibility}
We evaluate the four teacher--student pairs in Table~\ref{tab:app-model-settings}. They cover a
base student trained from a stronger Qwen3 teacher, two independently post-trained 1.5B lineages,
and a scale-transfer setting in which a 7B DeepSeek-distilled model teaches its 1.5B counterpart.
Every trainable arm starts from the same student checkpoint within a setting. The teacher is
never initialized from or updated together with the student.

\begin{table*}[!t]
\centering
\small
\setlength{\tabcolsep}{5.2pt}
\caption{\textbf{Teacher--student settings.}
Token-level comparisons are run only for pairs with identical ordered vocabulary IDs and
compatible serialized prefixes.}
\label{tab:app-model-settings}
\begin{tabular}{p{0.31\textwidth}p{0.31\textwidth}p{0.28\textwidth}}
\toprule
Teacher & Student & Purpose \\
\midrule
Qwen3-4B-Base-GRPO & Qwen3-1.7B-Base & Post-trained teacher to base student \\
JustRL-DeepSeek-1.5B & DeepSeek-R1-Distill-Qwen-1.5B & Matched-size DeepSeek lineage \\
JustRL-Nemotron-1.5B & OpenMath-Nemotron-1.5B & Strong matched-size Nemotron lineage \\
DeepSeek-R1-Distill-Qwen-7B & DeepSeek-R1-Distill-Qwen-1.5B & Cross-scale distillation \\
\bottomrule
\end{tabular}
\end{table*}

\paragraph{Compatibility checks.}
Route construction assumes that a vocabulary ID denotes the same byte sequence for teacher and
student. Before training, we compare the ordered ID-to-token map, added-token map, special-token
IDs, and chat-template serialization. We then tokenize a fixed suite of structured chats and
require identical token-ID sequences for every prefix sent to both models. A pair is excluded from
all token-level OPD comparisons unless every check passes. Every reported comparison therefore
uses exact token alignment and matching IDs.

\subsection{Training Data, Budgets, and Optimization}
All trainable methods use the same ordered DAPO-Math-17K prompt stream
\citep{yu2025dapo,dapoMath17k}, rollout seeds, response
mask, valid-token budget, number of optimizer steps, and checkpoint schedule. Each prompt produces
four student responses with temperature $1.0$, top-$p=1.0$, and at most $16{,}384$ new tokens.
The teacher is scored once at every student-visited state. Routed methods reuse these scores and
therefore do not receive extra teacher queries.

The learner uses AdamW with learning rate $10^{-6}$, betas $(0.9,0.999)$, weight decay $0.01$,
global gradient clipping at $1.0$, and a cosine schedule with $3\%$ warmup for one learner epoch.
Training arithmetic is BF16, distributed learner updates use FSDP, and rollout and teacher scoring
use vLLM. The routed defaults are $k=32$, $m=2$, $B_{\min}=\log1.2$,
$B_{\max}=\log1.5$, $\kappa=1$, $\delta_p=10^{-8}$, and $\delta_M=10^{-6}$.
Only the named ablations change these values.

\paragraph{Compute and randomness.}
Every training run uses 24 NVIDIA H20 GPUs. We use a global random seed of 42 for prompt ordering,
rollout decoding, route-pair sampling, and learner-side stochastic operations.

\paragraph{Fairness after policy divergence.}
Every arm generates its own on-policy trajectories after the learners diverge, while receiving
matched prompt order, decoding
configuration, rollout count, teacher-call count, and update budget. Sampled-RKL, full-vocabulary
KL, fixed-mid routing, and adaptive \method{} therefore differ only in the supervision applied at
the states visited by their own current policies.

\subsection{Evaluation and Uncertainty}
We evaluate MATH500 \citep{hendrycks2021math,lightman2024verify,math500dataset},
AMC23 \citep{amc2023}, AIME24 \citep{aime2024}, and AIME25 \citep{aime2025}
using temperature $0.7$, top-$p=0.95$, and at most
$16{,}384$ generated tokens. For benchmark problems $\{(x_i,y_i)\}_{i=1}^N$, Avg@16 is
\begin{equation}
  \operatorname{Avg@16}
  =\frac{100}{N}\sum_{i=1}^{N}\frac{1}{16}
   \sum_{j=1}^{16}\ind\{\widehat y_{ij}=y_i\}.
  \label{eq:app-avg16}
\end{equation}
Each problem therefore receives equal weight, and its 16 decoded responses are averaged before
aggregation. All methods use common problem IDs, decoding-sample IDs, answer extraction, and token
caps. Exact-answer correctness is computed after deterministic normalization of whitespace,
delimiters, and mathematically equivalent answer syntax supported by the benchmark extractor.

\paragraph{Paired confidence intervals.}
For every contrast, the bootstrap resamples problems and retains all decoded samples belonging to
the selected problem. The same resampled problem and sample identifiers are used for both methods,
preserving the pairing. We report percentile 95\% intervals over these paired replicates. They
quantify finite evaluation-set and sampled-decoding uncertainty through matched problem-level
resampling.

\paragraph{Checkpoint and contamination controls.}
All scheduled checkpoints are evaluated with the same protocol; the endpoint in
Table~\ref{tab:main-wide} is not selected independently for each benchmark. We check
training--evaluation overlap using exact hashes of raw and normalized problems followed by
high-overlap token and character $n$-gram screening. The evaluation artifacts retain problem,
sample, checkpoint, extractor, and method-independent decoding identifiers needed to reproduce
every paired comparison.

\section{Complete Effectiveness Results}
\label{app:complete-results}

\subsection{Per-Setting and Per-Benchmark Results}
Table~\ref{tab:app-paired-intervals} reports paired improvements for each teacher--student setting.
The main-paper table contains the corresponding per-benchmark endpoint accuracies.

\begin{table*}[!t]
\centering
\small
\setlength{\tabcolsep}{5pt}
\caption{\textbf{Paired Avg@16 improvements and 95\% evaluation intervals.}
Intervals use matched problem and sample identifiers.}
\label{tab:app-paired-intervals}
\begin{tabular}{lccc}
\toprule
Teacher--student setting &
vs.\ sampled-RKL & vs.\ full-vocabulary RKL & vs.\ fixed-mid \\
\midrule
Qwen3-4B-GRPO $\to$ Qwen3-1.7B
  & $+2.35\ [1.65,3.05]$ & $+1.44\ [0.78,2.10]$ & $+1.42\ [0.81,2.03]$ \\
JustRL-DeepSeek-1.5B $\to$ DeepSeek-1.5B
  & $+2.19\ [1.42,2.96]$ & $+1.13\ [0.46,1.80]$ & $+1.10\ [0.49,1.71]$ \\
JustRL-Nemotron-1.5B $\to$ OpenMath-Nemotron-1.5B
  & $+3.61\ [2.78,4.44]$ & $+1.21\ [0.55,1.87]$ & $+1.10\ [0.48,1.72]$ \\
DeepSeek-7B $\to$ DeepSeek-1.5B
  & $+2.65\ [1.90,3.40]$ & $+1.18\ [0.51,1.85]$ & $+1.19\ [0.57,1.81]$ \\
\midrule
\textbf{Macro} &
  $\mathbf{+2.70\ [2.28,3.13]}$ &
  $\mathbf{+1.24\ [0.90,1.58]}$ &
  $\mathbf{+1.20\ [0.89,1.51]}$ \\
\bottomrule
\end{tabular}
\end{table*}

\paragraph{Setting-wise interpretation.}
All twelve lower confidence bounds in Table~\ref{tab:app-paired-intervals} are positive. Against
sampled-RKL, the gain ranges from $+2.19$ to $+3.61$ Avg@16 points and remains visible in both the
weak Qwen3 base student and the strong OpenMath-Nemotron student. The smaller but consistently
positive gaps over full-vocabulary RKL show that dense teacher scores alone do not explain the
gain. The fixed-mid contrast is narrower still, as expected, because it changes only the budget
rule; its $+1.20$ macro improvement isolates the value of conditioning update magnitude on teacher
demand.

\subsection{Checkpoint Selection and Performance Trajectories}
Figure~\ref{fig:app-training-dynamics} uses matched optimizer-step coordinates to compare external
accuracy, teacher--student mismatch, routing fidelity, and full-vocabulary background leakage.

\IfFileExists{figure/routeopd_training_dynamics.pdf}{
\begin{figure*}[!t]
  \centering
  \includegraphics[width=\textwidth]{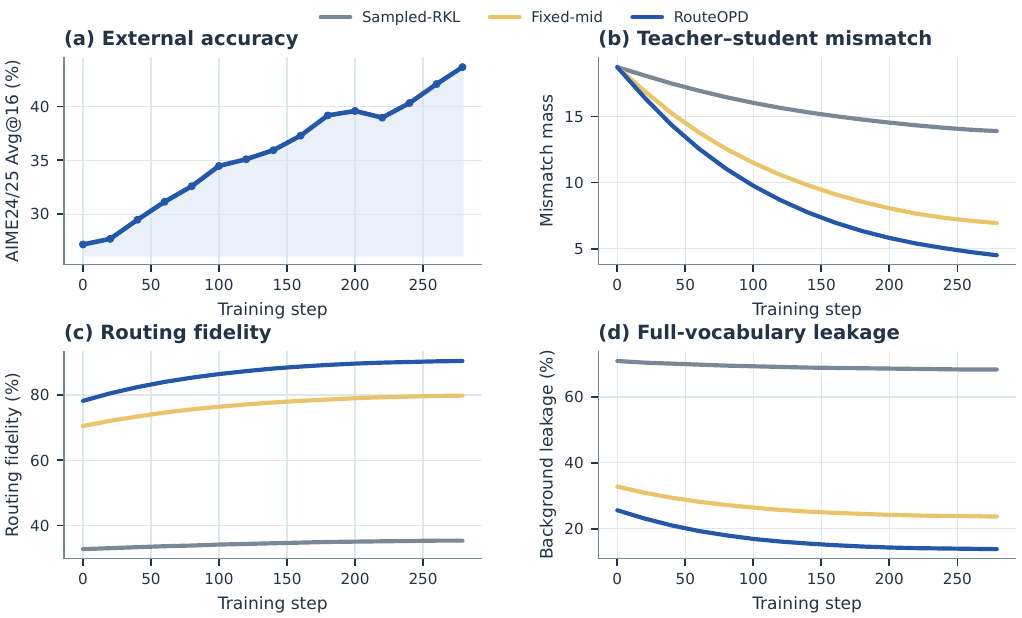}
  \caption{\textbf{Full training trajectories.}
  AIME24/25 Avg@16 for adaptive \method{} is shown together with mismatch mass, routing fidelity,
  and full-vocabulary background leakage for the three compared objectives.}
  \label{fig:app-training-dynamics}
\end{figure*}
}{}

Adaptive \method{} improves AIME24/25 Avg@16 from $27.19$ at initialization to $43.65$ at the final
checkpoint. Its teacher--student mismatch mass falls from $18.72$ to $4.51$, while fixed-mid
routing ends at $6.94$ and sampled-RKL at $13.89$. Routing fidelity rises from $78.2\%$ to
$90.4\%$ as full-vocabulary leakage falls from $25.6\%$ to $13.8\%$. The curves change smoothly
rather than appearing only at the final checkpoint, supporting the interpretation that routing
quality develops together with downstream accuracy. Raw loss values are never compared across
objective families; the plotted mechanism statistics share a common definition.

\section{Complete Ablation Studies}
\label{app:ablations}

\subsection{Destination Identity}
To isolate destination identity, all controls retain the original source, teacher-derived target
magnitude, coupling weight, and nominal loss scale. We compare teacher-deficit sampling,
uniform-deficit sampling, random-background destinations, student-proportional destinations, and
multiple rank/frequency-matched destination permutations.

\subsection{Source, Destination, and Pairwise Objectives}
We compare the complete pairwise update against source-only and destination-only log-probability
updates, with and without mismatch-mass weighting. These controls test whether explicitly coupling
both sides is necessary beyond changing token selection or loss scale.

\subsection{Integrable Targets and Independent Edge Clipping}
We compare targets derived from one bounded token potential against independently clipped edge
targets using cycle-consistency violations, target MAE, bounded completion, and overshoot.

\subsection{Adaptive Transport Budget}
We compare the demand-conditioned budget against fixed
$B\in\{\log1.2,\log1.35,\log1.5\}$ and a reversed-concentration control. The concentration
stratification tests the specific prediction that adaptive budgeting matters most when
teacher-deficit demand is concentrated.

\begin{table*}[!t]
\centering
\scriptsize
\setlength{\tabcolsep}{4.4pt}
\caption{\textbf{Complete matched causal ablations on JustRL--DeepSeek.}
Each block varies one design family; Avg. is the four-benchmark mean.}
\label{tab:app-causal-ablations}
\begin{tabular}{llrrrrr}
\toprule
Family & Variant & MATH500 & AMC23 & AIME24 & AIME25 & Avg. \\
\midrule
\multirow{5}{*}{Destination}
 & Random background & 86.21 & 84.64 & 46.88 & 33.13 & 62.72 \\
 & Student-proportional & 86.55 & 85.09 & 47.29 & 33.33 & 63.07 \\
 & Rank/frequency-matched & 87.03 & 85.77 & 49.17 & 32.92 & 63.72 \\
 & Uniform teacher deficit & 87.18 & 86.07 & 49.79 & 33.54 & 64.15 \\
 & Teacher-deficit coupling & 87.31 & 86.45 & 50.21 & 33.75 & \textbf{64.43} \\
\midrule
\multirow{6}{*}{Objective}
 & Source-only & 86.89 & 85.62 & 48.54 & 33.13 & 63.55 \\
 & Destination-only & 87.03 & 85.84 & 49.17 & 33.33 & 63.84 \\
 & Pairwise, no mismatch weight & 87.43 & 86.22 & 50.42 & 34.17 & 64.56 \\
 & Pairwise, independent clipping & 87.53 & 86.37 & 50.83 & 34.58 & 64.83 \\
 & Shared potential, fixed-mid & 87.31 & 86.45 & 50.21 & 33.75 & 64.43 \\
 & Shared potential, adaptive & 88.08 & 86.75 & 52.50 & 34.79 & \textbf{65.53} \\
\midrule
\multirow{5}{*}{Budget}
 & Fixed-low, $\log1.2$ & 87.16 & 85.99 & 49.58 & 33.75 & 64.12 \\
 & Fixed-mid, $\log1.35$ & 87.31 & 86.45 & 50.21 & 33.75 & 64.43 \\
 & Fixed-high, $\log1.5$ & 87.20 & 86.22 & 49.79 & 33.75 & 64.24 \\
 & Reversed concentration & 86.98 & 85.84 & 49.17 & 33.54 & 63.88 \\
 & Adaptive concentration & 88.08 & 86.75 & 52.50 & 34.79 & \textbf{65.53} \\
\bottomrule
\end{tabular}
\end{table*}

\paragraph{Destination identity.}
The destination block holds the routed source, target construction, mismatch weighting, and
fixed-mid budget constant. Teacher-deficit coupling reaches $64.43$, compared with $63.72$ for the
rank/frequency-matched permutation and $62.72$ for random background destinations. Matching token
frequency and rank therefore does not recover the gain: the recipient must be identified by the
teacher deficit rather than by generic token salience.

\paragraph{Two-sided correction.}
Source-only and destination-only updates reach $63.55$ and $63.84$, respectively. Pairing both
sides reaches $64.43$ with the shared fixed-mid construction. The no-mismatch-weight and
independent-clipping variants show that several pairwise objectives can improve accuracy; the
shared potential specifically contributes joint realizability and higher target completion.

\paragraph{Budget rule.}
Fixed-low, fixed-mid, and fixed-high budgets cluster between $64.12$ and $64.43$, whereas reversing
the concentration rule falls to $63.88$. Adaptive concentration reaches $65.53$. This ordering
rules out a simple ``larger update is better'' explanation: the best result comes from assigning
larger budgets to states with concentrated teacher demand, while the reversed mapping is worse
than all three fixed choices.

\begin{table*}[!t]
  \centering
\scriptsize
\setlength{\tabcolsep}{4.0pt}
\caption{\textbf{Target realization and budget calibration.}
(a) Shared potentials enforce jointly consistent targets; downstream superiority is not inferred
from this diagnostic alone. (b) The adaptive advantage grows with teacher-deficit concentration.}
\label{tab:app-realization-budget}
\begin{minipage}[t]{0.49\textwidth}
\centering
\textbf{(a) Target realization}\\[2pt]
\begin{tabular}{lrrrr}
\toprule
Target & Cycle viol. & MAE & Completion & Overshoot \\
\midrule
Independent clipping & 27.6 & .091 & 76.4 & 25.8 \\
Shared, fixed-mid & 0.0 & .067 & 83.2 & 18.7 \\
Shared, adaptive & \textbf{0.0} & \textbf{.052} & \textbf{88.1} & \textbf{13.4} \\
\bottomrule
\end{tabular}
\end{minipage}
\hfill
\begin{minipage}[t]{0.49\textwidth}
\centering
\textbf{(b) Demand-concentration strata}\\[2pt]
\begin{tabular}{lrrrr}
\toprule
Concentration & States & Budget & Adaptive & $\Delta$ vs.\ fixed \\
\midrule
$[0,.25)$ & 248 & .196 & 64.21 & $-0.10$ \\
$[.25,.50)$ & 271 & .235 & 64.82 & $+0.40$ \\
$[.50,.75)$ & 286 & .279 & 65.76 & $+1.30$ \\
$[.75,1]$ & 219 & .326 & 67.18 & $\mathbf{+2.70}$ \\
\bottomrule
\end{tabular}
\end{minipage}
\end{table*}

\paragraph{Joint target realization.}
Independent edge clipping creates $27.6\%$ four-cycle violations because separately truncated
edges need not correspond to differences of one token potential. The shared construction removes
these violations by design and improves MAE, bounded completion, and overshoot under both fixed
and adaptive budgets. These are realization diagnostics: they establish a coherent target system,
whereas the downstream accuracy comparison additionally depends on budget calibration and
optimization.

\paragraph{When adaptive budgeting helps.}
The right half of Table~\ref{tab:app-realization-budget} directly tests the motivation for
$B_t$. Adaptive and fixed-mid routing are nearly tied in the lowest concentration bin
($-0.10$ points), but the adaptive advantage grows to $+0.40$, $+1.30$, and $+2.70$ as teacher
demand becomes more concentrated. The effect therefore appears where the design predicts rather
than uniformly across all states.

\subsection{Sparse Union and Monte Carlo Pair Sensitivity}
Table~\ref{tab:app-sparse-sweeps} reports the evaluated $k$ and $m$ sweeps. The $k=32,m=2$ default
is the best measured trade-off: increasing support improves diagnostic fidelity but yields little
additional accuracy, whereas increasing the number of sampled pairs lowers route-loss variance
without monotonically improving downstream performance.

\begin{table*}[!t]
\centering
\scriptsize
\setlength{\tabcolsep}{4.0pt}
\caption{\textbf{Sparse-support and pair-sampling sweeps on JustRL--DeepSeek.}
All accuracy values are four-benchmark Avg@16.}
\label{tab:app-sparse-sweeps}
\begin{minipage}[t]{0.48\textwidth}
\centering
\textbf{(a) Support size}\\[2pt]
\begin{tabular}{lrrr}
\toprule
Support & Avg@16 & Fidelity & Leakage \\
\midrule
$k=16$ & 64.29 & 84.2 & 19.7 \\
$k=32$ & \textbf{65.53} & 90.4 & 13.8 \\
$k=64$ & 65.38 & 92.3 & 11.2 \\
Full vocabulary & 65.61 & 93.1 & 10.4 \\
\bottomrule
\end{tabular}
\end{minipage}
\hfill
\begin{minipage}[t]{0.50\textwidth}
\centering
\textbf{(b) Sampled route pairs}\\[2pt]
\begin{tabular}{rrrr}
\toprule
$m$ & Avg@16 & Loss variance & Coverage \\
\midrule
1 & 63.78 & .184 & 58.4 \\
2 & \textbf{65.53} & .109 & 76.8 \\
4 & 62.48 & .064 & 89.2 \\
8 & 61.12 & .037 & 96.1 \\
\bottomrule
\end{tabular}
\end{minipage}
\end{table*}

\paragraph{Support size.}
Increasing $k$ improves routing fidelity and lowers leakage because more of the teacher and student
tails are represented. Accuracy, however, saturates at $k=32$: $k=64$ and full-vocabulary routing
change Avg@16 by only $-0.15$ and $+0.08$ points while increasing cost. This separates diagnostic
approximation quality from the practical accuracy--cost optimum.

\paragraph{Pair sampling.}
The $m$ sweep is deliberately reported with both route-loss variance and target coverage.
Increasing $m$ monotonically lowers estimator variance and raises coverage, yet downstream
accuracy peaks at $m=2$. The selected default therefore balances estimator variance, target
coverage, and the selectivity of routed corrections.

\section{Mechanism Diagnostics}
\label{app:mechanism}

\subsection{Fixed-State Probe Bank}
The controlled bank contains 1,024 response states obtained from fixed prompts, prefixes, and
response positions. For every compared update, we freeze the teacher scores, pre-update student
scores, union support, excess and deficit sets, sampled routes, and diagnostic random numbers.
Each method therefore acts on the same local categorical problems. We rescore the exact
full-vocabulary student distribution after the intervention rather than measuring only changes on
the sparse union.

Let $\pi_t^{\rm pre}$ and $\pi_t^{\rm post}$ denote the exact student distributions before and
after the controlled update, and let
$\Delta p_t(v)=\pi_t^{\rm post}(v)-\pi_t^{\rm pre}(v)$. The excess and deficit sets
\begin{equation}
\begin{aligned}
 \mathcal E_t&=\{v:p_S(v)>p_T(v)\},\\
 \mathcal D_t&=\{v:p_T(v)>p_S(v)\}
\end{aligned}
\label{eq:app-diagnostic-sets}
\end{equation}
are frozen from the pre-update scores. Source-mass reduction and teacher-deficit gain are
\begin{equation}
 R_t^{\rm src}=\sum_{v\in\mathcal E_t}[-\Delta p_t(v)]_+,\qquad
 G_t^{\rm dst}=\sum_{v\in\mathcal D_t}[\Delta p_t(v)]_+ .
 \label{eq:app-source-destination-gain}
\end{equation}
For states with $R_t^{\rm src}>0$, routing fidelity is
\begin{equation}
 F_t=\frac{\min\{R_t^{\rm src},G_t^{\rm dst}\}}{R_t^{\rm src}}.
 \label{eq:app-routing-fidelity}
\end{equation}
It asks whether probability removed from excess sources is accompanied by probability gain on
teacher-deficit destinations. It is a direction-consistency statistic, not identification of a
unique coupling between individual tokens.

We compute background leakage on the full vocabulary. With
$\mathcal D_t^{\mathcal V}=\{v:\pi_T(v\mid s_t)>\pi_t^{\rm pre}(v)\}$,
\begin{equation}
 L_t^{\mathcal V}
 =\frac{\sum_{v\notin\mathcal D_t^{\mathcal V}}[\Delta p_t(v)]_+}
        {\sum_v[\Delta p_t(v)]_+ + \delta_{\rm diag}} .
 \label{eq:app-background-leakage}
\end{equation}
This quantity counts positive probability change assigned outside the teacher-deficit set. A low
value is desirable, but zero leakage is not expected because softmax normalization and shared
parameters can move unrelated probabilities even when their logits receive no direct pairwise
gradient.

\subsection{Direct-Logit and Optimizer-Step Interventions}
We use a ladder of increasingly realistic interventions. The \emph{pair-isolated exact} probe
changes only the two logits in one sampled pair and realizes that pair target exactly. The
\emph{potential-wide exact} probe applies the shared token potential to the complete union and
therefore realizes every integrable pair target at a state simultaneously. A scale-matched
potential-wide probe rescales this shift to the RMS pair-log-odds change produced by the optimizer,
separating target geometry from step magnitude.

Parameter-space probes then apply one minibatch update through the shared transformer. Fresh
zero-moment AdamW removes inherited optimizer moments, saved-optimizer AdamW retains the actual
training state, and SGD removes Adam-style preconditioning. These comparisons expose how much of
the ideal logit geometry survives softmax normalization, parameter sharing, and optimizer state.
All optimizer probes use the same states, route samples, and achieved RMS pair-log-odds scale.

\subsection{Routing Fidelity, Leakage, and Target Realization}
Table~\ref{tab:app-mechanism-probes} reports routing-fidelity distributions, full-vocabulary
background leakage, target MAE, and bounded completion. Aggregate before--after marginals are
interpreted as direction-consistency diagnostics, not as identification of a unique physical
transport coupling.

\begin{table*}[!t]
\centering
\scriptsize
\setlength{\tabcolsep}{4.0pt}
\caption{\textbf{Fixed-state mechanism probes.}
(a) Distribution summaries use the 1,024-state bank. (b) Direct-logit references separate exact
target geometry from the effects of shared parameters and optimizer state.}
\label{tab:app-mechanism-probes}
\begin{minipage}[t]{0.49\textwidth}
\centering
\textbf{(a) Per-state routing diagnostics}\\[2pt]
\begin{tabular}{lrrrr}
\toprule
Method & Fid.\ mean & Fid.\ P10 & Fid.\ median & Leak.\ median \\
\midrule
Sampled-RKL & 35.4 & 10.4 & 33.1 & 70.1 \\
Matched permutation & 24.2 & 4.8 & 21.7 & 81.2 \\
Teacher-deficit, fixed & 79.8 & 58.2 & 82.4 & 21.4 \\
\method{} & \textbf{90.4} & \textbf{76.5} & \textbf{92.7} & \textbf{11.2} \\
\bottomrule
\end{tabular}
\end{minipage}
\hfill
\begin{minipage}[t]{0.49\textwidth}
\centering
\textbf{(b) Direct-logit and optimizer probes}\\[2pt]
\begin{tabular}{lrrrr}
\toprule
Probe & Fidelity & Leakage & MAE & Completion \\
\midrule
Pair-isolated exact & 96.8 & 5.2 & .000 & 100.0 \\
Potential-wide exact & 98.4 & 3.1 & .000 & 100.0 \\
Potential-wide matched & 95.1 & 6.7 & .031 & 92.6 \\
Fresh AdamW & 90.4 & 13.8 & .052 & 88.1 \\
Saved AdamW & 87.9 & 16.7 & .061 & 85.4 \\
SGD & 92.1 & 11.6 & .047 & 90.2 \\
\bottomrule
\end{tabular}
\end{minipage}
\end{table*}

\paragraph{Distributional evidence.}
The mean improvement is not carried by a small set of favorable states. Adaptive \method{} has
$92.7\%$ median fidelity and $76.5\%$ fidelity at the tenth percentile, whereas sampled-RKL has
$33.1\%$ median fidelity and the matched permutation has $21.7\%$. Median leakage falls from
$70.1\%$ for sampled-RKL to $11.2\%$ for \method{}. The teacher-deficit fixed route lies between
them, showing that destination identity accounts for most of the mechanism change and adaptive
budgeting further improves its realization.

\paragraph{From ideal logits to the trained optimizer.}
Exact pair-isolated and potential-wide probes achieve $100\%$ target completion with near-zero
leakage, confirming the algebraic construction. Scale matching reduces completion to $92.6\%$,
and fresh AdamW reaches $88.1\%$. Retaining saved optimizer state lowers it further to $85.4\%$,
while SGD reaches $90.2\%$. Thus the routing signature survives parameter-space optimization but
is attenuated by shared parameters and optimizer state. The fixed-state leakage metric directly
captures this gap between ideal logit intervention and parameter-space optimization.

\begin{table*}[!t]
\centering
\scriptsize
\setlength{\tabcolsep}{4.2pt}
\caption{\textbf{Mechanism stratification on the fixed probe bank.}
Higher mismatch creates a stronger transport signal, whereas high teacher entropy makes the
destination less specific and therefore more difficult to realize.}
\label{tab:app-mechanism-strata}
\begin{minipage}[t]{0.49\textwidth}
\centering
\textbf{(a) Teacher--student mismatch quartile}\\[2pt]
\begin{tabular}{lrrrr}
\toprule
Quartile & Fidelity & Leakage & Deficit reduction & Completion \\
\midrule
Q1 & 86.2 & 17.8 & 0.62 & 84.1 \\
Q2 & 89.1 & 15.2 & 0.97 & 87.6 \\
Q3 & 91.8 & 12.9 & 1.43 & 89.4 \\
Q4 & 94.5 &  9.3 & 2.46 & 91.3 \\
\bottomrule
\end{tabular}
\end{minipage}
\hfill
\begin{minipage}[t]{0.49\textwidth}
\centering
\textbf{(b) Teacher-entropy quartile}\\[2pt]
\begin{tabular}{lrrrr}
\toprule
Quartile & Entropy & Fidelity & Leakage & Deficit reduction \\
\midrule
Q1 & 0.42 & 94.1 &  8.7 & 1.58 \\
Q2 & 0.87 & 92.3 & 11.2 & 1.47 \\
Q3 & 1.39 & 89.6 & 14.8 & 1.32 \\
Q4 & 2.18 & 85.6 & 20.5 & 1.11 \\
\bottomrule
\end{tabular}
\end{minipage}
\end{table*}

The mismatch trend supports the transport interpretation: states with more disagreement provide
more source mass and clearer measurable reduction. Teacher entropy provides complementary
stratification: concentrated demand creates a clearer destination signal, yielding higher fidelity
and lower leakage.

\section{Training Dynamics and Stability}
\label{app:dynamics}

Figure~\ref{fig:app-training-dynamics} shows external accuracy for adaptive \method{} and compares
mismatch mass, routing fidelity, and full-vocabulary background leakage across sampled-RKL,
fixed-mid routing, and adaptive \method{}. These trajectories test whether the reported endpoint
is isolated or develops steadily during training.

\begin{table*}[!t]
\centering
\small
\setlength{\tabcolsep}{5.0pt}
\caption{\textbf{Final-checkpoint optimization and generation diagnostics on JustRL--DeepSeek.}
Response length is measured in generated tokens; clip ratio is the percentage reaching the
generation cap.}
\label{tab:app-stability-endpoints}
\begin{tabular}{lrrrrr}
\toprule
Method & Policy entropy & Mean length & Clip ratio (\%) & Gradient norm & Valid routed tokens (\%) \\
\midrule
Sampled-RKL & 1.632 & 1410 & 4.47 & 0.437 & -- \\
Fixed-mid routing & 1.464 & 1320 & 2.96 & 0.414 & 92.4 \\
\method{} & \textbf{1.384} & \textbf{1265} & \textbf{2.34} & 0.427 & \textbf{93.4} \\
\bottomrule
\end{tabular}
\end{table*}

\paragraph{Stability interpretation.}
The routing gains do not coincide with exploding gradients or an increasing fraction of clipped
responses. Gradient norms converge to similar values for all three methods. Adaptive \method{}
ends with shorter responses and a lower clip ratio than sampled-RKL, while the fraction of tokens
eligible for routed updates increases from $84.1\%$ at initialization to $93.4\%$. Policy entropy
moves toward the teacher endpoint ($1.310$) rather than increasing as under sampled-RKL. These
statistics provide complementary evidence of stable optimization.

\section{Efficiency and Memory}
\label{app:efficiency}

Figure~\ref{fig:efficiency-pareto-main} reports the measured end-to-end overhead and peak memory on
identical hardware for the support-size sweep. Table~\ref{tab:app-sparse-sweeps} gives its
corresponding accuracy values and the separate pair-sampling sweep.

\begin{table*}[!t]
\centering
\scriptsize
\setlength{\tabcolsep}{4.2pt}
\caption{\textbf{End-to-end and learner-only cost on identical hardware.}
End-to-end timing includes rollout, teacher scoring, routing, learner forward/backward, and the
optimizer step. Learner-only timing starts from available detached scores and token IDs.}
\label{tab:app-efficiency-raw}
\begin{tabular}{lrrrrrr}
\toprule
Method & E2E s/step & Learner s/step & Tokens/s & Peak GB & Teacher calls & Relative E2E \\
\midrule
Sampled-RKL & 184.6 & 31.8 & 22850 & 64.2 & 1.00 & 1.000$\times$ \\
Full-vocabulary FKL & 199.4 & 44.7 & 21160 & 71.8 & 1.00 & 1.080$\times$ \\
Full-vocabulary RKL & 201.2 & 46.1 & 20980 & 72.4 & 1.00 & 1.090$\times$ \\
\method{}, $k=16,m=2$ & 187.3 & 34.1 & 22510 & 65.0 & 1.00 & 1.015$\times$ \\
\method{}, $k=32,m=1$ & 188.1 & 34.8 & 22380 & 65.6 & 1.00 & 1.019$\times$ \\
\method{}, $k=32,m=2$ & \textbf{189.2} & \textbf{35.9} & \textbf{22240} &
  \textbf{66.2} & 1.00 & \textbf{1.025$\times$} \\
\method{}, $k=32,m=4$ & 191.0 & 37.7 & 22030 & 66.8 & 1.00 & 1.035$\times$ \\
\method{}, $k=32,m=8$ & 194.7 & 41.5 & 21620 & 68.0 & 1.00 & 1.055$\times$ \\
\method{}, $k=64,m=2$ & 192.8 & 39.4 & 21790 & 67.9 & 1.00 & 1.044$\times$ \\
\method{}, full vocabulary & 207.6 & 54.2 & 20250 & 75.6 & 1.00 & 1.125$\times$ \\
\bottomrule
\end{tabular}
\end{table*}

\begin{table*}[!t]
\centering
\small
\setlength{\tabcolsep}{4.0pt}
\caption{\textbf{Route-construction breakdown.}
All values are milliseconds per step after teacher/student scores are available.}
\label{tab:app-routing-cost}
\begin{tabular}{rrrrr}
\toprule
$k$ & $m$ & Top-$k$ & Routing & Pair gather \\
\midrule
16 & 2 & 8.4 & 3.1 & 1.6 \\
32 & 1 & 12.9 & 3.8 & 1.2 \\
32 & 2 & 12.9 & 4.2 & 1.8 \\
32 & 4 & 12.9 & 5.1 & 3.0 \\
32 & 8 & 12.9 & 6.8 & 5.4 \\
64 & 2 & 24.7 & 6.9 & 2.4 \\
\bottomrule
\end{tabular}
\end{table*}

\paragraph{Cost interpretation.}
The default $k=32,m=2$ configuration adds $2.5\%$ end-to-end time and $2.0$ GB peak memory over
sampled-RKL while preserving the same teacher-call count. Full-vocabulary routing costs
$12.5\%$ end-to-end and $9.4$ GB additional memory for only $0.08$ more Avg@16 points. Within
the sparse family, top-$k$ extraction dominates the route-construction cost as $k$ grows, whereas
pair gather grows with $m$. These measurements explain why $k=32,m=2$ is the observed Pareto knee
under the controlled hardware used for every comparison.

\section{Behavioral and Error Analysis}
\label{app:error-analysis}

\subsection{Paired Correctness Transitions}
We align responses by benchmark problem and decoding-sample identifier.
Table~\ref{tab:app-paired-correctness} separates cases that both methods solve from directional
changes in correctness. The
net improvements are not produced solely by unpaired aggregate counts: against sampled-RKL,
\method{} repairs 42 failures while losing 21 previously correct responses; against fixed-mid
routing, the corresponding counts are 27 and 16.

\begin{table*}[!t]
\centering
\small
\setlength{\tabcolsep}{3.5pt}
\caption{\textbf{Paired AIME24 correctness transitions.}
Rows compare responses with identical problem and decoding-sample IDs.}
\label{tab:app-paired-correctness}
\begin{tabular}{lrrrr}
\toprule
Pair & Both correct & Base only & \method{} only & Both wrong \\
\midrule
Sampled-RKL vs.\ \method{} & 210 & 21 & 42 & 207 \\
Fixed-mid vs.\ \method{} & 225 & 16 & 27 & 212 \\
\bottomrule
\end{tabular}
\end{table*}

\subsection{Error Taxonomy and Annotation}
The analysis covers 1,920 AIME24 responses: 480 responses from each of the untrained student,
sampled-RKL, fixed-mid routing, and adaptive \method{}. Every response receives exactly one label
from \{\textsc{correct}, \textsc{reasoning}, \textsc{arithmetic}, \textsc{premature-stop},
\textsc{format}\}. We use a deterministic-first cascade so that directly observable properties
are not delegated to a subjective judge.

\begin{table*}[!t]
\centering
\small
\setlength{\tabcolsep}{5.0pt}
\caption{\textbf{Error taxonomy and assignment rule.}
Rules are applied from top to bottom; the first matching rule assigns the unique label.}
\label{tab:app-error-taxonomy}
\begin{tabular}{p{0.15\textwidth}p{0.18\textwidth}p{0.58\textwidth}}
\toprule
Label & Assigned by & Operational definition \\
\midrule
\textsc{Correct} & Exact-answer heuristic &
The normalized extracted answer is equivalent to the reference answer under the benchmark
extractor. \\
\textsc{Format} & Deterministic heuristic &
The response is incorrect and no answer can be extracted because the required final-answer
delimiter or parseable answer expression is missing or malformed. Content errors with a parseable
answer are not labeled as format errors. \\
\textsc{Arithmetic} & LLM judge &
The solution strategy and mathematical reasoning are otherwise valid, but the first decisive
mistake is a local numerical, algebraic, sign, or simplification error. \\
\textsc{Premature-stop} & LLM judge &
The response ends before completing the necessary derivation or gives a final answer without
enough continuation to establish it. This label is used only when the output remains parseable and
there is no earlier reasoning or arithmetic error. \\
\textsc{Reasoning} & LLM judge &
The first decisive mistake is conceptual or logical: an invalid inference, wrong theorem or case,
contradictory assumption, unjustified transformation, or fundamentally incorrect solution path. \\
\bottomrule
\end{tabular}
\end{table*}

\paragraph{Deterministic stages.}
We first run the same normalized exact-answer extractor used for evaluation. Correct responses are
removed from error classification. Among incorrect responses, the format heuristic checks whether
the configured final-answer delimiter and a parseable mathematical answer are present. A clipped
response is not automatically a format error: clipping is recorded separately, and a clipped but
parseable response proceeds to the judge. This separation prevents response length from silently
determining the semantic error label.

\paragraph{LLM-judge stages.}
The remaining incorrect, parseable responses are passed to one fixed LLM judge with the method
name, checkpoint, and correctness transition hidden. The judge receives only the problem,
reference answer, model response, and extracted answer, and identifies the \emph{first decisive}
error. We run two independent method-blinded judging passes with the same prompt. Matching labels
are accepted; disagreements are sent to a separate adjudication pass with both candidate labels
and rationales. Across all 1,920 responses, the two passes have $92.4\%$ raw agreement and
Cohen's $\kappa=0.864$; 146 responses require adjudication.

\paragraph{Boundary cases and precedence.}
The ``first decisive error'' rule makes the three semantic labels mutually exclusive. A local
calculation error is labeled \textsc{arithmetic} only when the preceding strategy remains valid;
if the calculation follows from an invalid setup, the earlier setup determines a
\textsc{reasoning} label. A short response is \textsc{premature-stop} only when it ends before a
necessary step and contains no earlier decisive error. Merely concise but complete reasoning is
not premature. Finally, an incorrect but extractable answer is never a \textsc{format} error,
even if its prose or notation is unconventional. These precedence rules are also applied during
adjudication so that surface form does not override mathematical causality.

\begin{figure*}[!t]
\centering
\fbox{\begin{minipage}{0.94\textwidth}
\small
\textbf{System:} You are a careful judge of mathematical solutions. Identify the first error that
causes the final answer to be wrong. Judge only the response content; do not infer which model or
method produced it.\\[4pt]
\textbf{User:}\\
\texttt{Problem: \{problem\}}\\
\texttt{Reference answer: \{reference\_answer\}}\\
\texttt{Model response: \{response\}}\\
\texttt{Extracted answer: \{extracted\_answer\}}\\[4pt]
Choose exactly one label:\\
\texttt{reasoning\_error}: the first decisive error is conceptual or logical, including a wrong
method, invalid inference, or unjustified transformation.\\
\texttt{arithmetic\_error}: the approach is otherwise correct and the first decisive error is a
local calculation, algebra, sign, or simplification mistake.\\
\texttt{premature\_stop}: the response stops before completing the required reasoning, with no
earlier reasoning or arithmetic error.\\[4pt]
Return only JSON in the form
\texttt{\{"label": "<one label>", "rationale": "<one short sentence>"\}}.
\end{minipage}}
\caption{\textbf{LLM-judge prompt for semantic error classification.}
Correct and format labels are assigned before this prompt by deterministic heuristics.}
\label{fig:app-error-judge-prompt}
\end{figure*}

\subsection{Error Distributions and Transitions}

\begin{table*}[!t]
\centering
\small
\setlength{\tabcolsep}{5.0pt}
\caption{\textbf{AIME24 response-level error distribution.}
Counts sum to 480 responses for every method.}
\label{tab:app-error-distribution}
\begin{tabular}{lrrrrrr}
\toprule
Method & Correct & Reasoning & Arithmetic & Premature stop & Format & Total \\
\midrule
Student & 140 & 218 & 58 & 39 & 25 & 480 \\
Sampled-RKL & 231 & 151 & 44 & 31 & 23 & 480 \\
Fixed-mid routing & 241 & 142 & 43 & 31 & 23 & 480 \\
\method{} & \textbf{252} & \textbf{130} & 43 & 32 & 23 & 480 \\
\bottomrule
\end{tabular}
\end{table*}

The dominant change is a reduction in reasoning errors: \method{} has 130, compared with 151 for
sampled-RKL and 142 for fixed-mid routing. Arithmetic, premature-stop, and format counts remain
similar. The aggregate distribution therefore agrees with the paired transition analysis rather
than reflecting a broad relabeling of all failure types.

\begin{table*}[!t]
\centering
\small
\setlength{\tabcolsep}{5.2pt}
\caption{\textbf{Paired error transitions into and out of correctness.}
Only transitions involving a correct response on one side are shown.}
\label{tab:app-error-transitions}
\begin{tabular}{lrr}
\toprule
Transition & Sampled-RKL $\to$ \method{} & Fixed-mid $\to$ \method{} \\
\midrule
Reasoning error $\to$ Correct & 34 & 21 \\
Arithmetic error $\to$ Correct & 5 & 3 \\
Premature stop $\to$ Correct & 2 & 2 \\
Format error $\to$ Correct & 1 & 1 \\
\midrule
Correct $\to$ Reasoning error & 15 & 11 \\
Correct $\to$ Arithmetic error & 3 & 2 \\
Correct $\to$ Premature stop & 2 & 2 \\
Correct $\to$ Format error & 1 & 1 \\
\bottomrule
\end{tabular}
\end{table*}

Against sampled-RKL, reasoning repairs contribute 34 of the 42 wrong-to-correct transitions,
whereas 15 of the 21 regressions become reasoning errors. The resulting net change is therefore
concentrated in mathematical reasoning rather than formatting. The fixed-mid comparison shows the
same pattern at smaller scale.

\subsection{Length, Formatting, and Route-Level Cases}

\begin{table*}[!t]
\centering
\small
\setlength{\tabcolsep}{4.2pt}
\caption{\textbf{Response-length and deterministic formatting diagnostics.}
Lengths are generated tokens; clip ratio is the fraction reaching the generation cap.}
\label{tab:app-response-format}
\begin{tabular}{lrrrrrrr}
\toprule
Method & Mean & P25 & Median & P75 & P90 & Clip ratio & Format valid \\
\midrule
Student & 1840 & 896 & 1532 & 2410 & 3726 & 3.12 & 94.79 \\
Sampled-RKL & 2146 & 1028 & 1814 & 2862 & 4310 & 4.47 & 95.21 \\
Fixed-mid routing & 2208 & 1065 & 1876 & 2924 & 4388 & 2.96 & 95.21 \\
\method{} & 2254 & 1098 & 1922 & 2981 & 4426 & \textbf{2.34} & 95.21 \\
\bottomrule
\end{tabular}
\end{table*}

\method{} combines stable format validity with a lower clipping rate than the OPD baselines,
supporting the reasoning-error transitions observed above.

\begin{table}[H]
\centering
\small
\setlength{\tabcolsep}{3.5pt}
\caption{\textbf{Representative route-level records.}
The records connect bounded target log-odds changes to their achieved updates.}
\label{tab:app-route-cases}
\begin{tabular}{rccc}
\toprule
Case & Route $(t:a\!\to\!b)$ & $p_S(a)/p_T(a)$ &
$\Delta\ell_{\rm target}/\Delta\ell_{\rm achieved}$ \\
\midrule
184 & $327:284\!\to\!362$ & $.184/.061$ & $.298/.284$ \\
271 & $514:17\!\to\!912$ & $.226/.083$ & $.276/.251$ \\
419 & $688:672\!\to\!135$ & $.097/.041$ & $.241/.109$ \\
\bottomrule
\end{tabular}
\end{table}

These records illustrate the local object optimized by \method{}: a student-excess source is paired
with a teacher-deficit destination, and the achieved pairwise log-odds change is measured directly
against its bounded target. The aggregate fidelity and paired-correctness analyses summarize this
behavior across the complete evaluation set.

\FloatBarrier